%% file: dkf.tex
\documentclass{article}
\pdfoutput=1
\usepackage{nips15submit_e,times}
\usepackage{hyperref}
\usepackage{url}
\usepackage{comment}
\usepackage{graphicx} %
\usepackage[font={small},compatibility=false]{caption}
\usepackage{subcaption}
\usepackage{algorithm}
\usepackage{algorithmic}
\usepackage{url}
\usepackage{cancel}
\usepackage{amsmath}
\usepackage{dsfont}
\usepackage{amsfonts} %
\usepackage{amsthm}
\usepackage{comment}
\usepackage{tikz-qtree}
\usepackage{listings}
\usepackage{bbm}
\usepackage{mathtools}
\usepackage{breqn}
\usepackage{todonotes}
\usepackage{rotating}
\usepackage{mathtools}
\usepackage{ amssymb }
\usepackage{natbib}
\usepackage{breqn}
\usepackage{booktabs}
\usepackage{multirow}

\newtheorem{theorem}{Theorem}[section]

\newcommand{\mathtext}[1]{\small{\textit{(#1)}}}
\newcommand{\Exp}[2]{\mathop{\mathbb{E}}_{#2}\left[#1\right]}
\newcommand{\vsigma}{\vec{\sigma}}
\DeclareMathOperator{\diag}{diag}
\DeclareMathOperator{\KL}{KL}

\newcommand{\pth}{p_\theta}
\newcommand{\qph}{q_\phi}
\newcommand{\Id}{\mathds{1}}
\newcommand{\dt}{\Delta}

\DeclareMathOperator{\Tr}{Tr}
\newcommand{\Prec}[1]{{\Sigma}_{#1}^{-1}}

\newcommand{\Normal}{\mathcal{N}}
\newcommand{\Ind}[1]{\mathbb{I}\left[#1\right]}

\newcommand{\deriv}[2]{\frac{\textit{\mathbf{d}}#1}{\textit{\mathbf{d}}#2}}
\newcommand{\vone}{\vec{\mathbf{1}}}
\newcommand{\vecx}{\vec{x}}
\newcommand{\vecu}{\vec{u}}
\newcommand{\vecz}{\vec{z}}
\newcommand{\bigCI}{\mathrel{\text{\scalebox{1.07}{$\perp\mkern-10mu\perp$}}}}
\newcommand{\R}{\mathbb{R}}
\newcommand{\meanfxn}{\textit{G}_{\alpha}}
\newcommand{\covfxn}{\textit{S}_{\beta}}
\newcommand{\emisfxn}{\textit{F}_{\kappa}}

\newcommand{\lthph}{\mathcal{L}(x;(\theta, \phi))}

\newcommand{\qIndep}{\textbf{q-INDEP}}
\newcommand{\qLR}{\textbf{q-LR}}
\newcommand{\qRNN}{\textbf{q-RNN}}
\newcommand{\qBRNN}{\textbf{q-BRNN}}

\usepackage{tikz}
\usetikzlibrary{bayesnet}

\theoremstyle{plain}

\title{Deep Kalman Filters}

\author{
Rahul G. Krishnan \qquad Uri Shalit \qquad David Sontag\\
Courant Institute of Mathematical Sciences\\
New York University\\
}

\nipsfinalcopy %

\begin{document}
\maketitle

\begin{abstract}
Kalman Filters are one of the most influential
models of time-varying phenomena. 
They admit an intuitive probabilistic interpretation, 
have a simple functional form,
and enjoy widespread adoption in a variety of disciplines. 
Motivated by recent variational methods for learning deep generative models, 
we introduce a unified algorithm to efficiently learn a broad spectrum of Kalman filters. 
Of particular interest is the use of temporal generative models for counterfactual inference. 
We investigate the efficacy of such models for counterfactual inference, and to that end we introduce
the ``Healing MNIST'' dataset where long-term structure, noise and  actions are applied to sequences of digits. 
We show the efficacy of our method for modeling this dataset. We further 
show how our model can be used for counterfactual inference for patients,  
based on electronic health record data of 8,000 patients over 4.5 years.
\end{abstract}

\setlength{\belowcaptionskip}{-5pt}
\input{introduction}
\input{background}
\input{related}

\input{model}
\input{learning}
\input{experiments}

\input{medical}

\input{discussion}

\section*{Acknowledgements}
The Tesla K40s used for this research were donated by the NVIDIA 
Corporation. The authors gratefully acknowledge support by the DARPA Probabilistic Programming for Advancing
Machine Learning (PPAML) Program under AFRL prime contract
no. FA8750-14-C-0005, ONR \#N00014-13-1-0646, and NSF CAREER award
\#1350965. 

\bibliographystyle{authordate1}
\small{
\bibliography{refs}
}

\clearpage
\appendix
\input{app_likelihood}
\input{factorized_kl}

\input{add_expts}

\end{document}

%% file: introduction.tex
\section{Introduction}

The compilation of Electronic Health Records (EHRs) 
is now the norm across hospitals in the United States.
A patient record may be viewed
as a sequence of diagnoses, surgeries, laboratory values and drugs prescribed over time. 
The wide availability of these records now allows us to apply machine learning techniques to answer 
medical questions: What is the best course of treatment for a patient? Between two drugs, 
can we determine which will save a patient? Can we find patients who are ``similar'' 
to each other? Our paper introduces new techniques for learning \emph{causal} generative temporal models from noisy high-dimensional data, that we believe is the first step towards addressing these questions.

We seek to model the change of the patient's state over time. We do this by learning a
representation of the patient that (1) evolves over time and (2) is sensitive to the effect of the actions taken by doctors. 
In particular, the approach we adopt is to learn a time-varying, generative model of patients.  

Modelling temporal data is a well studied problem in machine learning. 
Models such as the Hidden Markov Models (HMM), Dynamic Bayesian Networks (DBN), and Recurrent Neural Networks (RNN)
have been proposed to model the probability of sequences.
Here, we consider a widely used probabilistic model: Kalman filters \citep{kalman1960new}. 
Classical Kalman filters are linear dynamical systems, that have enjoyed remarkable success in the last few decades. 
From their use in GPS to weather systems and speech recognition models,
few other generative models of sequential data have enjoyed such widespread adoption across many domains. 

In classical Kalman filters, the latent state evolution as well as the 
emission distribution and action effects are modelled as linear functions perturbed by Gaussian noise. 
However, for real world applications the use of linear transition and emission distribution limits the capacity to model complex phenomena,
and modifications to the functional form of Kalman filters have been proposed. 
For example, the Extended Kalman Filter \citep{jazwinski2007stochastic} and the Unscented Kalman Filter \citep{wan2000unscented} 
are two different methods to learn temporal models with non-linear transition and emission distributions (see also \citet{roweis2000algorithm} and \citet{haykin2004kalman}).
The addition of non-linearities to the model makes learning more difficult. \citet{raiko2009variational} explored ways of using linear approximations 
and non-linear dynamical factor analysis in order to overcome these difficulties. However, 
their methods do not handle long-range temporal interactions and scale quadratically with the latent dimension.

We show that recently developed techniques in variational inference \citep{rezende2014stochastic,kingma2013auto} can be adopted
to learn a broad class of the Kalman filters that exist in the literature using a single algorithm. 
Furthermore, using deep neural networks, we can enhance
Kalman filters with arbitrarily complex transition dynamics and emission distributions. 
We show that we can tractably learn such models
by optimizing a bound on the likelihood of the data.

Kalman filters have been used extensively for optimal control, where the model attempts to capture 
how actions affect the observations, precipitating the task of choosing the best control signal towards a given objective. We use Kalman filters for a different yet closely related task: performing counterfactual inference.
In the medical setting, counterfactual inference attempts to model the effect of an intervention such as a surgery or a drug, on an outcome, e.g. whether the patient survived. The hardness of this problem lies in the fact that
typically, for a single patient, we only see one intervention-outcome pair (the patient cannot have taken and not taken the drug). 
The key point here is that by modelling the sequence of observations such as diagnoses and lab reports, as well as the interventions or actions (in the form of surgeries and drugs administered) across patients, we hope to learn the effect of interventions on  
a patient's future state. 

We evaluate our model in two settings. First we introduce ``Healing MNIST'', a dataset of perturbed, noisy and rotated MNIST digits. We show our model captures both short- and long-range effects of actions performed on these digits. Second, we use EHR data from $8,000$ diabetic and pre-diabetic patients gathered over 4.5 years.
We investigate various kinds of Kalman filters learned using our framework and 
use our model to learn the effect anti-diabetic medication has on patients.

The contributions of this paper are as follows:
\begin{itemize}
	\item  Develop a method for probabilistic generative modelling of sequences of complex observations, perturbed by non-linear actions, using deep neural nets as a building block. 
		We derive a bound on the log-likelihood
		of sequential data and an algorithm to learn a broad class of Kalman filters. 
	\item	We evaluate the efficacy of different recognition distributions for inference and learning.
	\item	We consider this model for use in counterfactual inference 
		with emphasis on the medical setting.
		To the best of our knowledge, the use of continuous state space models 
		has not been considered
		for this goal. 
		On a synthetic setting we empirically validate that our model is able to capture
		patterns within a very noisy setting and model the effect of external actions. On real patient data we show that our model can successfully perform counterfactual inference to show the effect of anti-diabetic drugs on diabetic patients.
\end{itemize}

%% file: background.tex
\section{Background}\label{sec:back}

\textbf{Kalman Filters}
Assume we have a sequence of unobserved variables $z_1, \ldots, z_T \in \R^{s} $. For each unobserved variable $z_t$ we have a corresponding \emph{observation} $x_t \in \R^d$, and a corresponding \emph{action} $u_t \in \R^c$, which is also observed. In the medical domain, the variables $z_t$ might denote the true state of a patient, the observations $x_t$ indicate known diagnoses and lab test results, and the actions $u_t$ correspond to prescribed medications and medical procedures which aim to change the state of the patient. The classical Kalman filter models the observed sequence $x_1, \ldots x_T$ as follows:
$$z_{t} = G_{t} z_{t-1} +  B_{t} u_{t-1} + \epsilon_{t} \, \textit{ (action-transition)\, ,}\qquad x_{t} = F_{t} z_t + \eta_t \, \textit{ (observation)}, \\$$
where $\epsilon_t \sim \mathcal{N}(0,\Sigma_t)$, $\eta_t \sim \mathcal{N}(0, \Gamma_t)$ are zero-mean i.i.d. normal random variables, with covariance matrices which may vary with $t$.
This model assumes that the latent space evolves linearly, transformed at time $t$ by the state-transition matrix $G_{t} \in \R^{s \times s}$. The effect of the control signal $u_t$ is an additive linear transformation of the latent state obtained by adding the vector $B_t u_{t-1}$, where $B_t \in \R^{s \times c}$ is known as the \emph{control-input model}. Finally, the observations are generated linearly from the latent state via the observation matrix $F_t \in \R^{d \times s}$. 

In the following sections, we show how to replace all the linear transformations with 
non-linear transformations parameterized by neural nets. 
The upshot is that the non-linearity makes learning much more challenging, as the posterior 
distribution  $p(z_1, \ldots z_T| x_1, \ldots, x_T, u_1,\ldots, u_T)$ becomes intractable to compute.

\textbf{Stochastic Backpropagation}
In order to overcome the intractability of posterior inference, we make use of recently introduced variational autoencoders \citep{rezende2014stochastic,kingma2013auto}
to optimize a variational lower bound on the model log-likelihood. 
The key technical innovation is the introduction of a \emph{recognition network}, a neural network which approximates the intractable posterior.

Let $p(x,z) = p_0(z) \pth(x|z)$ be a generative model for the set of observations $x$, where $p_0(z)$ is the prior on $z$ and $\pth(x|z)$ is a generative model parameterized by $\theta$. In a model such as the one we posit, the posterior distribution $\pth(z|x)$ is typically intractable. Using the well-known variational principle, we posit an approximate posterior distribution $\qph(z|x)$, also called a \emph{recognition model} - see Figure \ref{fig:vae}.
We then obtain the following lower bound on the marginal likelihood:
\begin{align}\label{eqn:varlowbnd}
 \log \pth(x) 
&=\log \int_{z} \frac{\qph(z|x)}{\qph(z|x)} \pth(x|z) p_0(z) \mathrm{d}z \geq \int_z  \qph(z|x)\log \frac{ \pth(x|z)p_0(z)}{\qph(z|x)} \mathrm{d}z  \nonumber \\
&= \Exp{\log \pth(x|z)}{\qph(z|x)} - \KL(\, \qph(z|x) || p_0(z)\, ) = \lthph,
\end{align}
where the inequality is by Jensen's inequality.  
Variational autoencoders aim to maximize the lower bound using a parametric model $\qph$ conditioned on the input. 
Specifically, \citet{rezende2014stochastic,kingma2013auto} both  suggest using a neural net to parameterize $\qph$, such that $\phi$ are the parameters of the neural net. 
The challenge in the resulting optimization problem is that the lower bound \eqref{eqn:varlowbnd} includes an expectation w.r.t. $\qph$, which implicitly depends on the network parameters $\phi$. 
This difficulty is overcome by using \emph{stochastic backpropagation}: assuming that the latent state is normally distributed $\qph(z | x) \sim \mathcal{N}\left(\mu_\phi(x),\Sigma_\phi(x)\right)$, 
 a simple transformation allows one to obtain Monte Carlo estimates of the gradients of $\Exp{\log \pth(x|z)}{\qph(z|x)} $ with respect to $\phi$. 
 The $\KL$ term in \eqref{eqn:varlowbnd} can be estimated similarly since it is also an expectation. 
 If we assume that the prior $p_0(z)$ is also normally distributed, the $\KL$ and its gradients 
 may be obtained analytically.

\textbf{Counterfactual Estimation}
Counterfactual estimation is the task of inferring the probability of a result given different circumstances than those empirically observed. For example, in the medical setting, one is often interested in questions such as ``What would the patient's blood sugar level be had she taken a different medication?''. Knowing the answers to such questions could lead to better and more efficient healthcare. We are interested in providing better answers to this type of questions, by leveraging the power of large-scale Electronic Health Records. 

\citet{pearl2009causality} framed the problem of counterfactual estimation in the language of graphical models and \emph{do}-calculus. 
If one knows the graphical model of the variables in question, then for some structures estimation of counterfactuals is possible by setting a variable of interest (e.g. medication prescribed) to a given value and performing inference on a derived sub-graph.
In this work, we do not seek to learn the true underlying causal graph structure but rather seek to use \emph{do}-calculus to 
observe the effect of interventions under a causal interpretation of the model we posit. 

%% file: related.tex
\section{Related Work}
The literature on sequential modeling and Kalman filters is vast and  
here we review some of the relevant work on the topic
with particular emphasis on recent work in machine learning. We point the reader 
to \cite{haykin2004kalman} for a summary of some approaches to learn Kalman filters.

\citet{mirowski2009dynamic} model sequences using dynamic factor graphs with an EM-like procedure 
for energy minimization. \cite{srivastava2015unsupervised} consider unsupervised
learning of video representations with LSTMs. They encode a sequence in a fixed length hidden
representation of an LSTM-RNN and reconstruct the subsequent sequence based on this representation.
\cite{gregor2015draw} consider a temporal extension to variational autoencoders where independent
latent variables perturb the hidden state of an RNN across time. 

\cite{langford2009learning} adopt a different approach to learn nonlinear dynamical systems using black-box
classifiers. 
Their method relies on learning three sets of classifiers. The first is trained to construct a compact representation $s_t$ 
to predict the $x_{t+1}$ from $x_t$, the second uses $s_{t-1}$ and $x_{t-1}$ to predict $s_t$. The third
trains classifiers to use $s_{<t}$ to predict $s_t$ and consequently $x_t$. 
In essence, the latent space $s_t$ is constructed using these classifiers. 

\citet{ganSigmoid} similarly learn a generative model
by maximizing a lower bound on the likelihood of sequential
data but do so in a model with discrete random variables. 

\citet{bayer2014learning} create a stochastic variant of Recurrent Neural Networks (RNNs) 
by making the hidden state of the RNN a function of stochastically sampled latent variables
at every time step. 
\citet{chung2015recurrent} model sequences of length $T$
using $T$ variational autoencoders. They use a single RNN
that (1) shares parameters 
in the inference and generative network and 
(2) models the parameters of the prior and approximation to the posterior at time $t\in[1,\ldots T]$ 
as a deterministic function
of the hidden state of the RNN. 
There are a few key differences between their work and ours. 
First, they do not model the effect of external actions
on the data, and second, their choice of model 
ties together inference and sampling from the model whereas we consider decoupled generative and recognition networks. Finally, the time varying ``memory'' of their resulting generative model is both deterministic and 
stochastic whereas ours is entirely stochastic. i.e our model retains the Markov
Property and other conditional independence statements held by Kalman filters.

Learning Kalman filters with 
Multi-Layer Perceptrons was considered by \citet{raiko2009variational}. 
They approximate the posterior using non-linear dynamic factor analysis \citep{valpola2002unsupervised}, which scales quadratically with the latent dimension.
Recently, \citet{watter2015embed}  
use temporal generative models 
for optimal control.
While \citeauthor{watter2015embed} aim to learn a locally linear latent dimension within which to perform optimal control, 
our goal is different: we wish to model the data in order to perform counterfactual inference.
Their training algorithm relies on approximating the bound on the likelihood by training
on consecutive pairs of observations.

In general, control applications deal with domains where the effect of action is instantaneous, unlike in the medical setting. In addition, most control scenarios involve a setting such as controlling a robot arm where the control signal has a major effect on the observation; we contrast this with the medical setting where medication can often have a weak impact on the patient's state, compared with endogenous and environmental factors.

For a general introduction to estimating expected counterfactual effects over a population - see \cite{morgan2014counterfactuals,hofler2005causal,rosenbaum2002observational}. For insightful work on counterfactual inference, in the context of a complex machine-learning and ad-placement system, see \citet{bottou2013counterfactual}.

Recently, \cite{finale2013pomdp} use a partially observable Markov process for modeling diabetic patients over time, finding that the latent state corresponds to relevant lab test levels (specifically, A1c levels).

%% file: model.tex
\section{Model}

Our goal is to fit a generative model to a sequence of observations and actions, motivated by the nature of patient health record data. We assume that the observations come from a latent state which evolves over time. We assume the observations are a noisy, non-linear function of this latent state. Finally, we also assume that we can observe actions, which affect the latent state in a possibly non-linear manner. 

Denote the sequence of observations $\vecx = (x_1, \ldots, x_T)$ and actions $\vecu = (u_1, \ldots, u_{T-1})$, with corresponding latent states $\vecz = (z_1, \ldots, z_T)$.  As previously, we assume that $x_t \in \R^d$, $u_t \in \R^c$, and $z_t \in \R^s$. The generative model for the deep Kalman filter is then given by:

\begin{equation}
	\begin{split}
	\label{eqn:gen_model}
	& z_1 \sim \mathcal{N}(\mu_0;\Sigma_0) \\
	& z_t \sim \mathcal{N} (\meanfxn(z_{t-1},u_{t-1},\dt_t),\covfxn(z_{t-1},u_{t-1},\dt_t))
\quad \\
	& x_t \sim \Pi(\emisfxn(z_t)).
	\end{split}
\end{equation}

That is, we assume that the distribution of the latent states is Normal, with a mean and covariance which are nonlinear functions of the previous latent state, the previous actions, and the time different $\dt_t$ between time $t-1$ and time $t$ \footnote{More precisely, this is a \emph{semi-Markov} model, and we assume that the time intervals are modelled separately. In our experiments we consider homogeneous time intervals.}.
The observations $x_t$ are distributed according to a distribution $\Pi$ (e.g. a Bernoulli distribution if the data is binary) whose parameters are a function of the corresponding latent state $z_t$.
Specifically, the functions $\meanfxn,\covfxn,\emisfxn$ are assumed to be  parameterized by deep neural networks. We set $\mu_0 = 0$, $\Sigma_0 = I_d$, and therefore we have that $\theta=\{\alpha,\beta,\kappa\}$ are the parameters of the generative model. We use a diagonal covariance matrix $\covfxn(\cdot)$, and employ a log-parameterization, thus ensuring that the covariance matrix is positive-definite. The model is presented in Figure \ref{fig:dkf}, along with the recognition model $\qph$ which we outline in Section \ref{sec:learnmodel}.

The key point here is that Eq. \eqref{eqn:gen_model} subsumes a large family of linear and non-linear latent space models. By restricting the functional 
forms of $\meanfxn,\covfxn,\emisfxn$, we can train different kinds of Kalman filters
within the framework we propose. 
For example, by setting $\meanfxn(z_{t-1},u_{t-1}) = G_t z_{t-1}+B_t u_{t-1},\covfxn = \Sigma_t,\emisfxn=F_t z_{t}$ where $G_t,B_t,\Sigma_t,F_t$ are 
matrices, we obtain classical Kalman filters. In the past, modifications to the Kalman filter typically introduced a new learning algorithm
and heuristics to approximate the posterior more accurately. In contrast, within the framework we propose any parametric differentiable function can be substituted
in for one of $\meanfxn,\covfxn,\emisfxn$. Learning any such model can be done using backpropagation as will be detailed in the next section. 

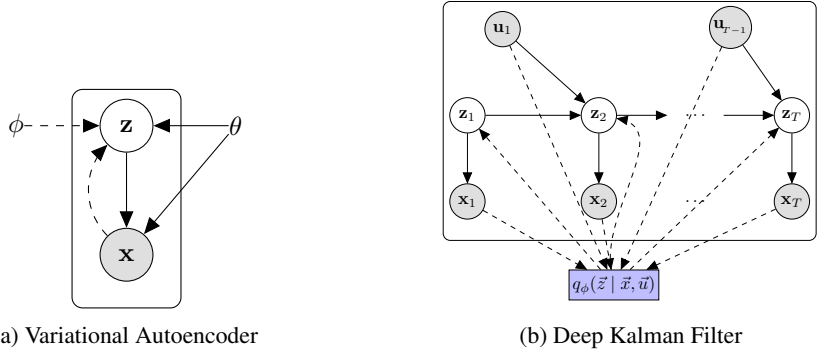
\begin{figure}[h!]
    \begin{subfigure}[b]{0.3\textwidth}
    \centering
	\begin{tikzpicture}[scale=1, transform shape]
		\node[obs] (x1) {$\mathbf{x}$};
		\node[latent, above=of x1] (z1) {$\mathbf{z}$};
		\node[const, left=of z1] (phi1) {$\phi$};
		\node[const, right=of z1] (theta1) {$\theta$};
		\edge [dashed] {phi1} {z1};
		\edge {theta1} {z1};
		\edge {theta1} {x1};
		\draw (x1) edge[out=135,in=225,->,dashed] (z1);
		\edge {z1} {x1};
		\plate [xscale=1.5] {} {(x1)(z1)} {} ;
	\end{tikzpicture}
	\caption{Variational Autoencoder}
	\label{fig:vae}
    \end{subfigure}
    \begin{subfigure}[b]{0.65\textwidth}
    \centering
    	\includegraphics[width=0.6\textwidth]{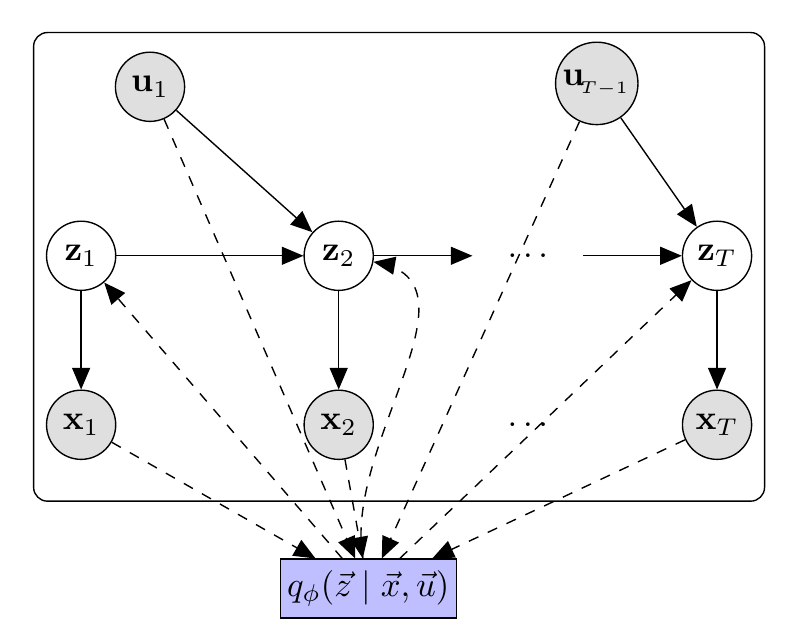}
	\caption{Deep Kalman Filter}
	\label{fig:dkf}
    \end{subfigure}
    \caption{\small{(a): Learning static generative models. Solid lines denote the generative model $p_0(z) \pth(x|z)$, dashed lines denote the variational approximation $q_{\phi}(z|x)$ to the intractable posterior $p(z|x)$. The variational parameters $\phi$ are learned jointly with the generative model parameters $\theta$.
    (b): Learning in a Deep Kalman Filter. A parametric approximation to $\pth(\vecz|\vecx)$, denoted $\qph(\vecz|\vecx,\vecu)$, is used to perform inference during learning.}}
    \label{fig:plate_models}
\end{figure}

%% file: learning.tex
\section{Learning using Stochastic Backpropagation}\label{sec:learnmodel}

\subsection{Maximizing a Lower Bound}

We aim to fit the generative model parameters $\theta$ which maximize the conditional likelihood of the data given the external actions, i.e we desire 
$\max_{\theta} \log \pth(x_1\ldots,x_T|u_1\ldots u_{T-1})$.
Using the variational principle, we apply the lower bound on the log-likelihood of the observations $\vecx$ derived in Eq. \eqref{eqn:varlowbnd}. 
We consider the extension of the Eq. \eqref{eqn:varlowbnd} to the temporal setting where we use the following factorization of the prior:

\begin{equation}
	\label{eqn:q_fact}
	\qph(\vecz|\vecx,\vecu) = \prod_{t=1}^T q(z_t|z_{t-1},x_t,\ldots,x_T,\vecu)
\end{equation}
We motivate this structured factorization of $\qph$ in Section \ref{sec:opt_q}. 
We condition the variational approximation not just on the inputs $\vecx$ but also on the actions $\vecu$.

Our goal is to derive a lower bound to the conditional log-likelihood in a form that will factorize easily and make learning more amenable. 
The lower bound in Eq. \eqref{eqn:varlowbnd} has an analytic form of the $\KL$ term only for the simplest of transition models $\meanfxn,\covfxn$. 
Resorting to sampling for estimating the gradient of the KL term results in very high variance. 
Below we show another way to factorize the KL term which results in more stable gradients, by using the Markov property of our model.

We have for the conditional log-likelihood (see Supplemental Section A for a more detailed derivation):
\begin{align*}
&\log \pth(\vecx|\vecu) %
\geq
\int_{\vecz} \qph(\vecz|\vecx,\vecu) \log \frac{p_0(\vecz|\vecu)\pth(\vecx|\vecz,\vecu)}{\qph(\vecz|\vecx,\vecu)} d\vecz\\ 
	&= \Exp{\log \pth(\vecx|\vecz,\vecu)}{\qph(\vecz|\vecx,\vecu)} - \KL(\qph(\vecz|\vecx,\vecu)||p_0(\vecz|\vecu)) \\
	&\stackrel{\mathtext{using $x_{t}\bigCI x_{\neg t}|\vecz$}}{=} \sum_{t=1}^T\Exp{\log \pth(x_t|z_t,u_{t-1})}{z_t \sim \qph(z_t|\vecx,\vecu)} - \KL(\qph(\vecz|\vecx,\vecu)||p_0(\vecz|\vecu))  = \lthph.
\end{align*}
The KL divergence can be factorized as:
\begin{dmath}
	\label{eqn:KLdiv_factorized}
	\KL(\qph(\vecz|\vecx,\vecu)||p_0(\vecz)) \newline\\
	= \int_{z_1}\ldots\int_{z_T} \qph(z_1|\vecx,\vecu)\ldots \qph(z_T|z_{T-1},\vecx,\vecu) 
	\log \frac{p_0(z_1,\cdots,z_T)}{\qph(z_1|\vecx,\vecu)\ldots \qph(z_T|z_{T-1},\vecx,\vecu)} \mathrm{d}\vecz\\
	\stackrel{\mathtext{Factorization of $p(\vecz)$}}{=} \KL(\qph(z_1|\vecx,\vecu)||p_0(z_1))\newline\\ 
	+ \sum_{t=2}^{T} \Exp{\KL(\qph(z_t|z_{t-1},\vecx,\vecu)||p_0(z_t|z_{t-1},u_{t-1}))}{z_{t-1} \sim \qph(z_{t-1}|\vecx,\vecu)}\nonumber.
\end{dmath}

This yields: 
\begin{align}
\label{eqn:bound_likelihood}
&	\log \pth(\vecx|\vecu) \geq  \lthph = \nonumber \\
& \sum_{t=1}^T\Exp{\log \pth(x_t|z_t)}{\qph(z_t|\vecx,\vecu)} - \KL(\qph(z_1|\vecx,\vecu)||p_0(z_1))\nonumber\\ 
&-  \sum_{t=2}^{T} \Exp{\KL(\qph(z_t|z_{t-1},\vecx,\vecu)||p_0(z_t|z_{t-1},u_{t-1}))}{\qph(z_{t-1}|\vecx,\vecu)}.
\end{align}
Equation \eqref{eqn:bound_likelihood} is differentiable
in the parameters of the model ($\theta,\phi$), and we can apply backpropagation for updating $\theta$, and stochastic backpropagation 
for estimating the gradient w.r.t. $\phi$ of the expectation terms w.r.t. $\qph(z_t)$.
Algorithm \ref{alg1} depicts the learning algorithm. It can be viewed as a four stage process. The first stage is inference of $\vecz$ from an input $\vecx$, $\vecu$ by the recognition network $\qph$.
The second stage is having the generative model $\pth$ reconstruct the input using the current estimates of the posterior. The third stage involves estimating gradients of the likelihood with respect to $\theta$ and $\phi$,
and the fourth stage involves updating parameters of the model. Gradients are typically averaged across stochastically sampled
mini-batches of the training set. 

\begin{algorithm}[t]
\caption{Learning Deep Kalman Filters}
\begin{algorithmic} \label{alg1}
\WHILE{$notConverged()$}
\STATE $\vecx \leftarrow sampleMiniBatch()$
\STATE Perform inference and estimate likelihood:
\STATE 1. $\hat{z}\sim \qph(\vecz|\vecx,\vecu)$
\STATE 2. $\hat{x}\sim \pth(\vecx|\hat{z})$
\STATE 3. Compute $\nabla_{\theta}\mathcal{L}$ and $\nabla_{\phi}\mathcal{L}$ (Differentiating \eqref{eqn:bound_likelihood}) 
\STATE 4. Update $\theta,\phi$ using ADAM 
\ENDWHILE 
\end{algorithmic}
\end{algorithm}

\subsection{On the choice of the Optimal Variational Model\label{sec:opt_q}}
For time varying data, there exist many choices for the recognition network. 
We consider four variational models of increasing complexity. 
Each model conditions on a different subset of the observations through the use of Multi-Layer Perceptrons (MLP) and Recurrent Neural Nets (RNN) (As implemented in \cite{zaremba2014learning}):
\begin{itemize}
	\item \qIndep: $q(z_t|x_t,u_t)$ parameterized by an MLP
	\item \qLR:  $q(z_t|x_{t-1},x_t,x_{t+1},u_{t-1},u_t,u_{t+1})$ parameterized by an MLP
	\item \qRNN: $q(z_{t}|x_1,\ldots,x_t, u_1, \ldots u_t)$ parameterized by a RNN
	\item \qBRNN: $q(z_{t} | x_1,\ldots,x_T,u_1, \ldots, u_T)$ parameterized by a bi-directional RNN 
\end{itemize} 
In the experimental section we compare the performance of these four models on a challenging sequence reconstruction task.

An interesting question is whether the Markov properties of the model can enable better design of approximations to the posterior. 

\begin{theorem}
	\label{thm:p_fact}
	For the graphical model depicted in Figure \ref{fig:dkf}, the posterior factorizes as: 
	\begin{equation*} p(\vecz|\vecx,\vecu) = p(z_1|\vecx,\vecu)\prod_{t=2}^T p(z_t|z_{t-1},x_t,\ldots,x_T,u_{t-1},\ldots,u_{T-1})\end{equation*}
\end{theorem}
\begin{proof}
	We use the independence statements implied by the generative model in Figure \ref{fig:dkf}
	to note that $p(\vecz|\vecx,\vecu)$, the true posterior, factorizes as:
	\begin{equation*} p(\vecz|\vecx,\vecu) = p(z_1|\vecx,\vecu)\prod_{t=2}^T p(z_t|z_{t-1},\vecx,\vecu) \end{equation*}
	Now, we notice that $z_t\bigCI x_1,\ldots,x_{t-1}|z_{t-1}$ and $z_t\bigCI u_1\ldots,u_{t-2}|z_{t-1}$, yielding:
	$$p(\vecz|\vecx,\vecu) = p(z_1|\vecx,\vecu)\prod_{t=2}^T p(z_t|z_{t-1},x_t,\ldots,x_T,u_{t-1},\ldots,u_{T-1})$$
\end{proof}

The significance of Theorem \ref{thm:p_fact} is twofold. First, it tells us how we can 
use the Markov structure of our graphical model to simplify the posterior
that any $\qph(\vecz)$ must approximate. Second, it yields insight on how to design
approximations to the true posterior. Indeed this motivated the factorization of $\qph$ in Eq. \ref{eqn:q_fact}.
Furthermore, instead of using a bi-directional RNN
to approximate $p(z_t|\vecx,\vecu)$ by summarizing both the past and the future ($x_1,\ldots,x_T$), 
one can approximate
the same posterior distribution using a single RNN that summarizes the future ($x_t,\ldots,x_T$) as 
long as one also conditions on the previous latent
state ($z_{t-1}$). Here, $z_{t-1}$ serves as a summary of $x_1,\ldots,x_{t-1}$. 

For the stochastic backpropagation model, the variational lower bound is tight if and only if $\KL(\qph(z|x)||\pth(z|x))=0$. In that case, we have that $\lthph = \log \pth(x)$, and the optimization objective \eqref{eqn:bound_likelihood} reduces to a maximum likelihood objective.
In the stochastic backpropagation literature, the variational distribution $\qph(z|x)$ is usually Gaussian and therefore cannot be expected to be equal to $\pth(z|x)$. An interesting question is whether using the idea of the universality of normalizing flows  \citep{tabak2010density,rezende2015variational} one can transform $\qph(z|x)$ to be equal (or arbitrarily close) to $\pth(z|x)$ and thus attain equality in the lower bound. Such a result leads to a consistency result for the learned model, stemming from the consistency of maximum likelihood.

\subsection{Counterfactual Inference}
Having learned a generative temporal model, we can use the model to perform counterfactual inference. 
Formally, consider a scenario where we are interested in evaluating the effect of an intervention
at time $t$. 
We can perform inference on the set of observations: $\{x_1,\ldots,x_t,u_1,\ldots,u_{t-1}\}$
using the learned $\qph$. This gives us an estimate $z_t$. At this point, we can apply $u_t$ (the action intended for the patient)
as well as $\tilde{u}_t$ (the action to be contrasted against). We can forward sample from this latent state in order to contrast the expected effect of different actions.

%% file: experiments.tex
\section{Experimental Section}
We implement and train models in Torch \citep{collobert2011torch7} using ADAM \citep{kingma2014adam} with a learning rate of $0.001$ to perform gradient ascent. 
Our code is implemented to parameterize $\log \covfxn$ during learning.
We use a two-layer Long-Short Term Memory Recurrent Neural Net (LSTM-RNN, \cite{zaremba2014learning}) for sequential variational models. 
We regularize models during
training (1) using dropout \citep{srivastava2014dropout} with a noise of $0.1$ to the input
of the recognition model
(2) through the addition of small random uniform noise (on the order of a tenth of the maximal value) to the actions. 

\textbf{Comparing recognition models}
We experiment with four choices of variational models of increasing complexity: $\qIndep$ where $q(z_t|x_t)$ is parameterized by an MLP, $\qLR$ where $q(z_t|x_{t-1},x_t,x_{t+1})$ is parameterized by an MLP, $\qRNN$ where $q(z_{t}|x_1,\ldots,x_t)$ is parameterized by an RNN, and $\qBRNN$ where $q(z_{t} | x_1,\ldots,x_T)$ is parameterized by a bi-directional RNN.

\subsection{Healing MNIST}
Healthcare data exhibits diverse structural properties. Surgeries and drugs
vary in their effect as a function of patient age, gender, ethnicity and comorbidities. 
Laboratory measurements are often noisy, and diagnoses
may be tentative, redundant or delayed. In insurance health claims data, the situation 
is further complicated by arcane, institutional specific
practices that determine how decisions made by healthcare professions are 
repurposed into codes used for reimbursements. 

To mimic learning under such harsh conditions, we consider a synthetic dataset derived from the MNIST Handwritten Digits \citep{lecun2010mnist}.
We select several digits and create a synthetic dataset where rotations are performed to the digits. The rotations are encoded as the actions ($\vecu$)
and the rotated images as the observations ($\vecx$). This realizes a sequence of rotated images. To each such generated training sequence, 
exactly one sequence of three consecutive squares is superimposed with the top-left 
corner of the images in a random starting location, and add up to 20\% bit-flip noise.  We consider two experiments: {\bf{Small Healing MNIST}}, using a single example of the digit $1$ and digit $5$, and {\bf{Large Healing MNIST}} where 
$100$ different digits (one hundred 5's and one hundred 1's) are used. The training set comprises approximately $40000$ sequences of length five for  {\bf{Small Healing MNIST}}, and $140000$ sequences of length five for {\bf{Large Healing MNIST}}.
The large dataset represents the temporal evolution of two distinct subpopulations 
of patients (of size $100$ each). The squares within the sequences  are intended to be analogous to seasonal flu or other ailments
that a patient could exhibit that are independent of the actions and which last several timesteps.

The challenges present within this dataset are numerous. 
(1) Image data is intrinsically high dimensional and much recent work has focused on learning
patterns from it. It represents a setting where the posterior is complex and often requires 
highly non-linear models in order to approximate it. 
The additions of rotated images to the training data adds more complexity to the task.
(2) In order to learn through random noise that is this high, 
one needs to have a model of sequences capable of performing ``filtering''. 
Models that rely on predicting the next image based on the previous
one \citep{goroshin2015learning,memisevic2013learning} may not suffice to learn the structure of digits in the presence of large noise and rotation.  
Furthermore, long-range patterns - e.g. the three consecutive blocks in the upper-left corner - that exist in the data are beyond the scope of such models.

We learn models using the four recognition models described in Section \ref{sec:learnmodel}.
Figure \ref{fig:healingMNIST_reconstructions} shows examples of training sequences (marked {\bf{TS}}) from {\bf{Large Healing MNIST}} provided to the model, 
and their corresponding reconstructions (marked {\bf{R}}). 
The reconstructions are performed by feeding the input sequence into the learned recognition network, 
and then sampling from the resulting posterior distribution. Recall that the model posits $\vecx$ drawn 
from an independent Bernoulli distribution whose mean parameters (denoted mean probabilities) we visualize. We discuss results in more detail below.

\begin{figure}
    \centering
    \begin{subfigure}[b]{0.33\textwidth}
	\includegraphics[width=0.9\textwidth]{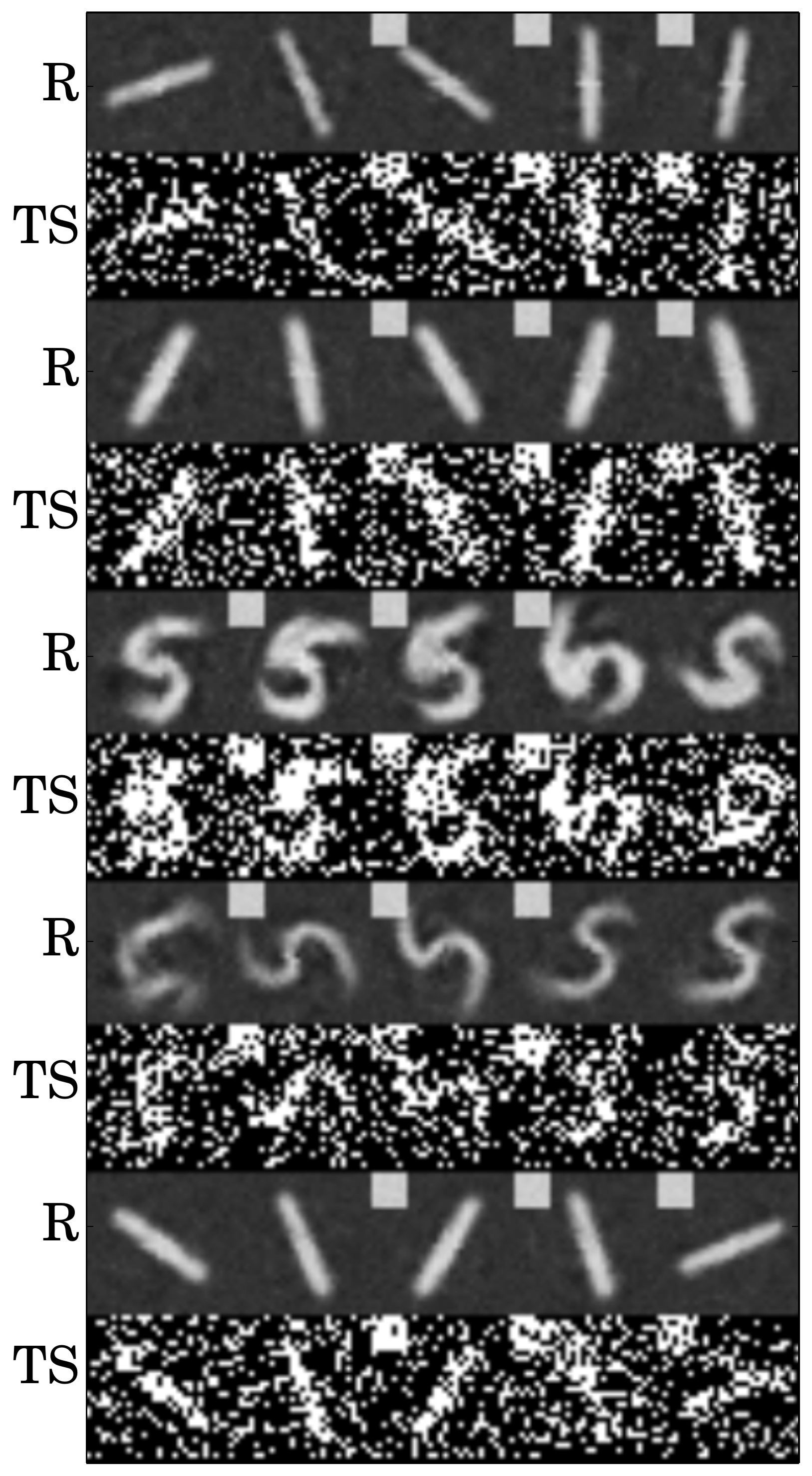}
	\caption{\small{Reconstruction during training}}
	\label{fig:healingMNIST_reconstructions}
    \end{subfigure}
    \begin{subfigure}[b]{0.3\textwidth}
	\includegraphics[width=0.9\textwidth]{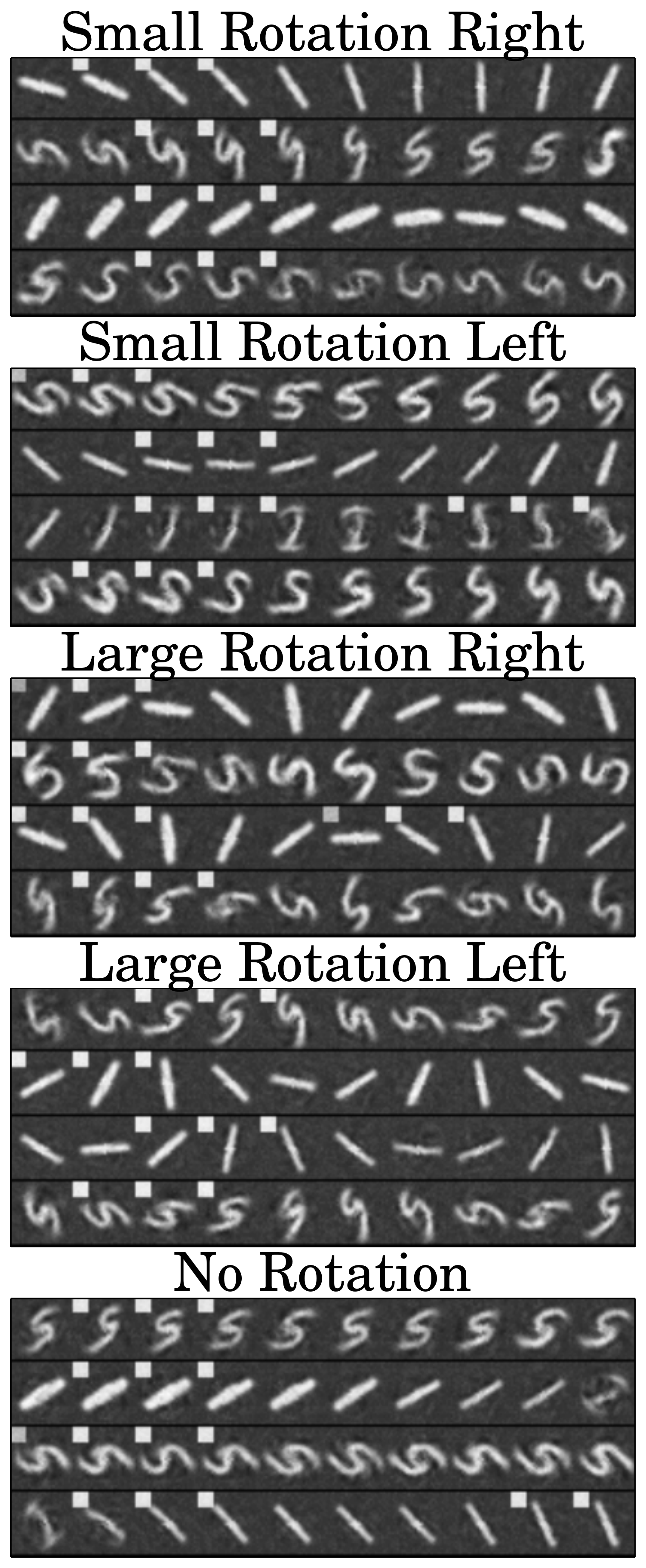}
	\caption{\small{Samples: Different rotations}}
	\label{fig:healingMNIST_samples}
    \end{subfigure}
    \begin{subfigure}[b]{0.33\textwidth}
	\includegraphics[width=0.9\textwidth]{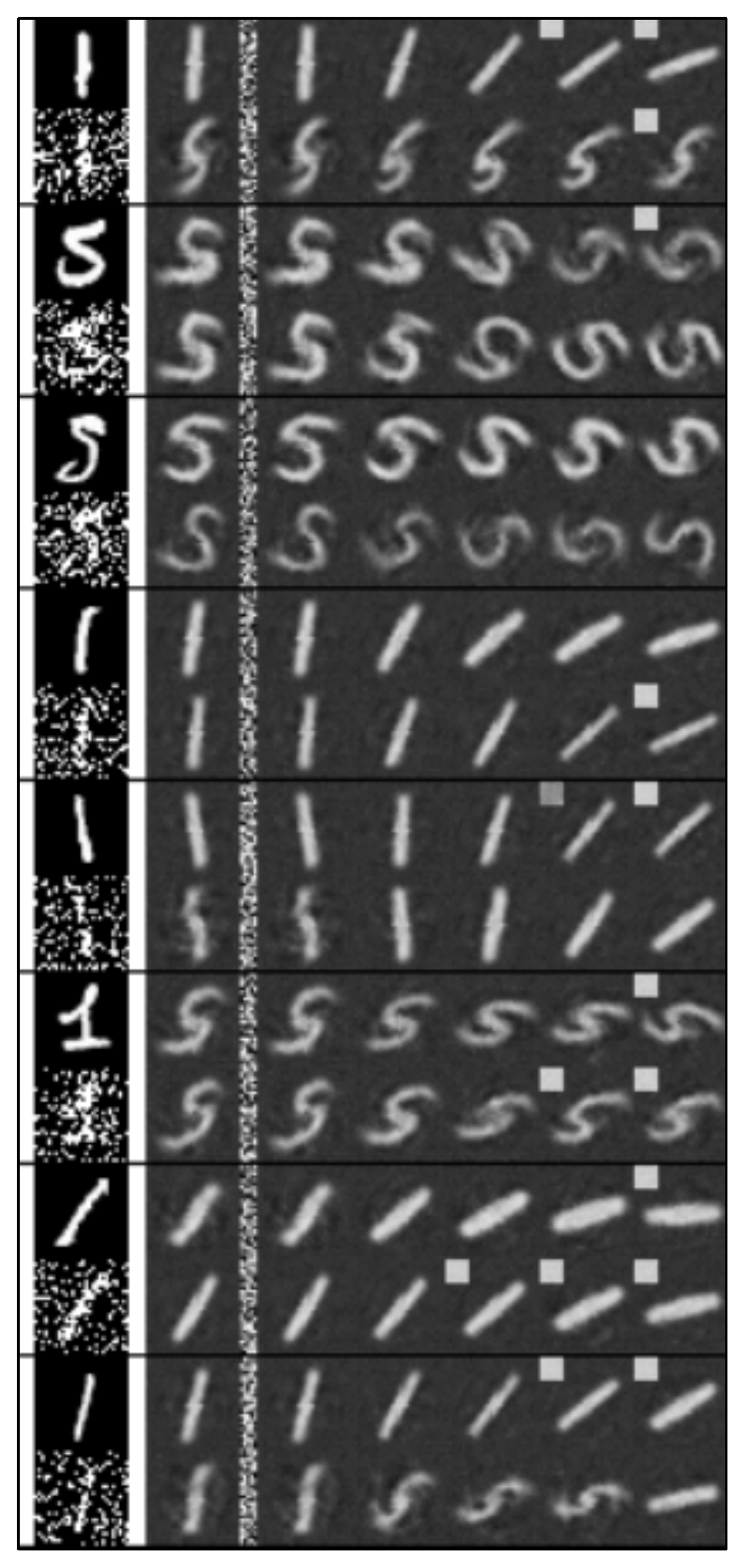}%
	\caption{\small{Inference on unseen digits}}
	\label{fig:healingMNIST_counterfactuals}
    \end{subfigure}
    \caption{\small{{\bf{Large Healing MNIST}}. (a) Pairs of Training Sequences (TS) and Mean Probabilities of Reconstructions (R) shown above. (b) Mean probabilities sampled from the model under different, constant rotations. (c) Counterfactual Reasoning. We reconstruct variants of the digits $5,1$  \emph{not} present in the training set, with (bottom) and without (top) bit-flip noise. We infer a sequence of 1 timestep and display the reconstructions from the posterior. We then keep the latent state and perform forward sampling and reconstruction from the generative model under a constant
	right rotation.}}
	\label{fig:healingMNIST}
\end{figure}
\begin{figure}
    \centering
    \begin{subfigure}[b]{0.55\textwidth}
	\includegraphics[width=0.85\textwidth]{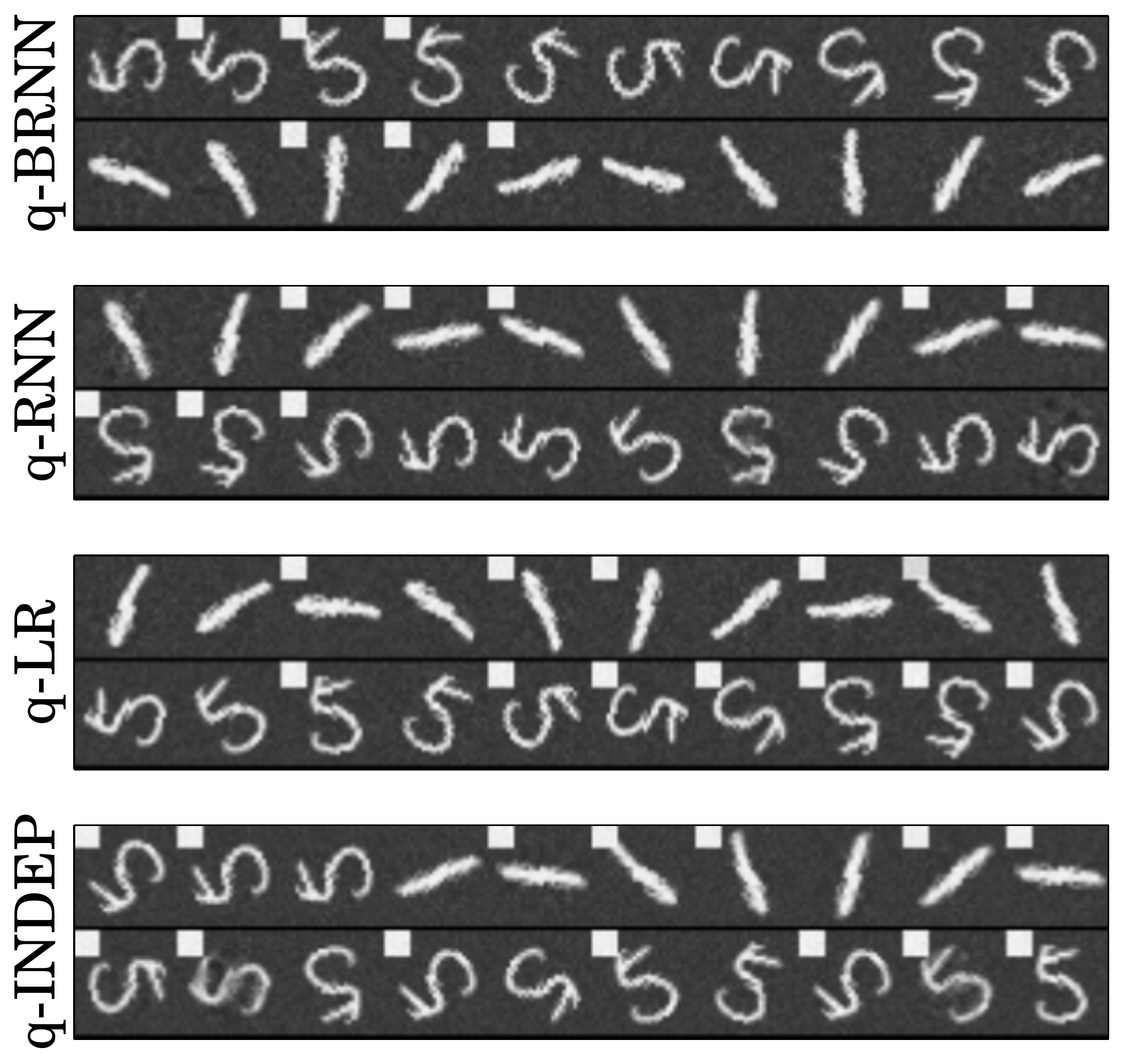}
	\caption{\small{Samples from models trained with different $\qph$}}
	\label{fig:qModel_samples}
    \end{subfigure}
    \begin{subfigure}[b]{0.44\textwidth}
	\includegraphics[width=0.85\textwidth]{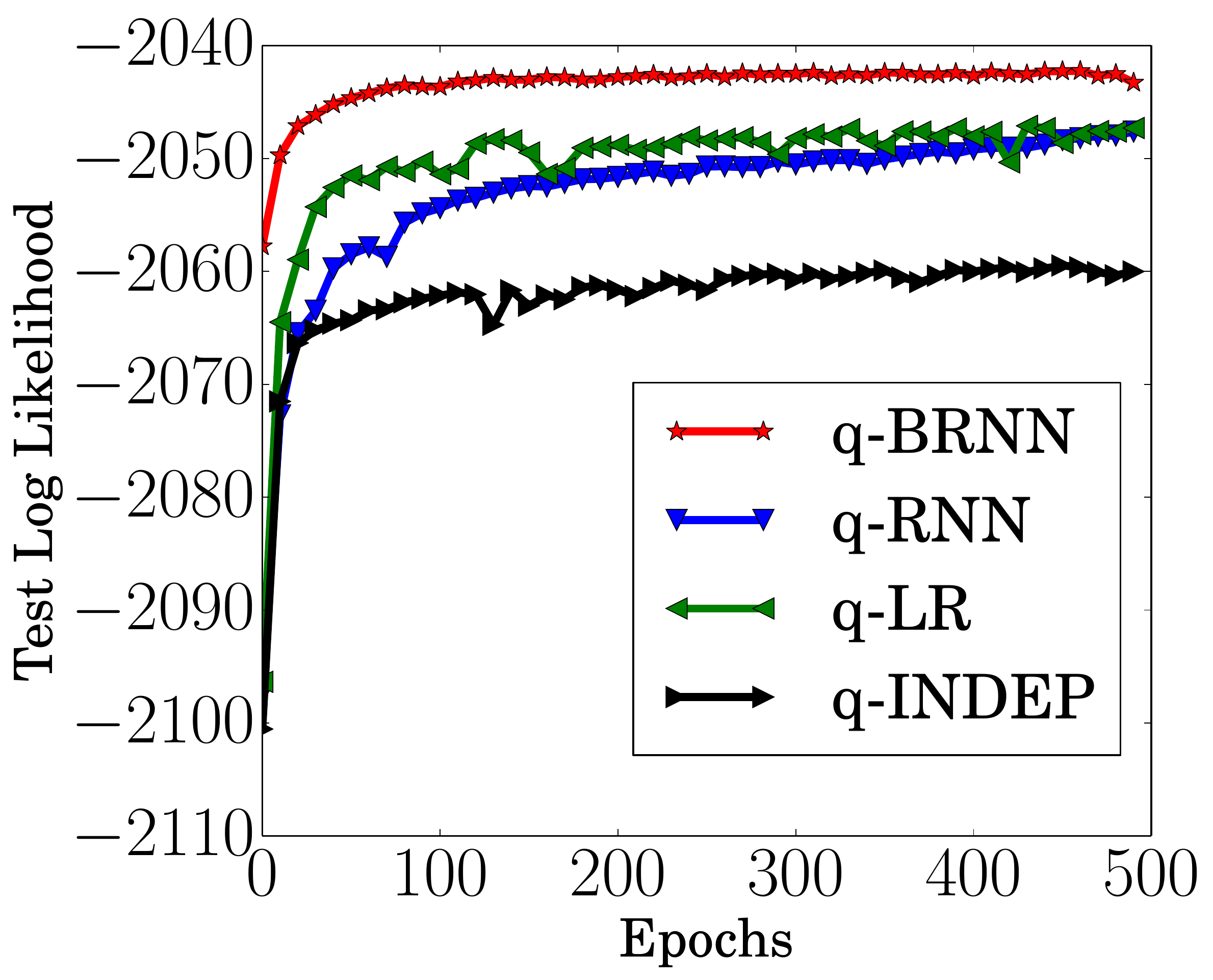}
	\caption{\small{Test Log-Likelihood for models trained with different $\qph$}}
	\label{fig:qModel_ll}
    \end{subfigure}
    
    \caption{\small{
	    {\bf{Small Healing MNIST}}: (a) Mean probabilities sampled under different variational models with a constant, large rotation applied to the right. (b) Test log-likelihood under different recognition models.
    }}
    \label{fig:qModel}
\end{figure}

\subsection*{\textbf{Small Healing MNIST}: Comparing Recognition Models}

We evaluated the impact of different variational models on learning by examining test log-likelihood and by visualizing the samples generated by the models. 
Since $\qRNN$ and $\qBRNN$ have more parameters by virtue of their internal structure,  
we added layers to the $\qIndep$ network (6 layers) and $\qLR$ network (4 layers) until training on the dataset was infeasible - i.e. did not make any gains after more than 100 epochs of training\footnote{In particular, we found that
adding layers to the variational model helped but only up to a certain point. Beyond that, learning the model was infeasible.}. Figure \ref{fig:qModel_ll} depicts the log-likelihood under the test set (we estimate test log-likelihood using importance sampling based on the recognition network - see Supplemental Section A). Figure \ref{fig:qModel_samples} depicts pairs of sequence samples obtained under each of the variational models. 

It is unsurprising that $\qBRNN$ outperforms the other variants. In a manner similar to the Forward-Backward algorithm, 
the Bi-Directional RNN summarizes the past and the future at every timestep to form the most effective approximation to the posterior distribution
of $z_t$. 

It is also not surprising that $\qIndep$ performs poorly, both in the quality of the samples and
in held out log-likelihood. Given the sequential nature of the data, the posterior for $z_t$ is poorly approximated when
only given access to $x_t$. 
The samples capture
the effect of rotation but not the squares. 

Although the test log-likelihood was better for $\qLR$ than for $\qRNN$, the samples obtained were worse off - in particular, they did not capture the three consecutive block structure.  
The key difference between $\qLR$ and $\qRNN$ is the fact that the former has no memory that it carries across time. This facet will likely be more relevant 
in sequences where there are multiple patterns and the recognition network needs to remember the order in which the patters are generated.

Finally, we note that both $\qRNN$ and $\qBRNN$ learn generative models with plausible samples.  

\subsection*{\textbf{Large Healing MNIST}}

Figure \ref{fig:healingMNIST_reconstructions} depicts pairs of training sequences, and the mean probabilities obtained
after reconstruction as learning progresses. There are instances (first and third from the top) where the noise level is too high for the 
structure of the digit to be recognized from the training data alone. 
The reconstructions also shows 
that the model learns different styles of the digits (corresponding to variances within individual patients).

Figure \ref{fig:healingMNIST_samples} depicts samples from the model under varying degrees of rotation (corresponding to the intensity of a treatment for example).
Here again, the model shows that it is capable of learning variations within the digit as well as realizing the effect of an action and its intensity. This is a simple form of counterfactual reasoning that can be performed by the model, since none of the samples on display are within the training set.  

Figure \ref{fig:healingMNIST_counterfactuals} answers two questions. The first is what happens when we ask the model to reconstruct on data that looks similar
to the training distribution but not the same. 
The input image is on the left (with a clean and noisy version of the digit displayed) and the following sample represent the reconstruction by the variational model of a sequence
created from the input images. Following this, we forward sample from the model using the inferred latent representation under a constant action.  

To this end, we consider digits of the same class (i.e. $1,5$) but of a different style than the model was trained on.
This idea has parallels within the medical setting where one asks about the course of action for a new patient. 
On this unseen patient, the model would infer a latent state similar to a patient that exists in the training set.
This facet of the model motivates further investigation into the model's capabilities as a metric for patient similarity. 
To simulate the medical question: What would happen if the doctor prescribed the drug ``rotate right mildly''? We forward sample
from the model under this action.

In most cases, noisy or not, the patient's reconstruction matches a close estimate of a digit found in the training set (close in log-likelihood since
this is the criterion the emission distribution was trained on). The final four rows depict scenarios in which the noise level is too high and
we don't show the model enough sequences to make accurate inferences about the digit.

%% file: medical.tex
\subsection{Generative Models of Medical Data}
\begin{figure}[h]
\begin{minipage}{0.4\textwidth}
\begin{subfigure}[b]{\linewidth}
	\includegraphics[width=\linewidth]{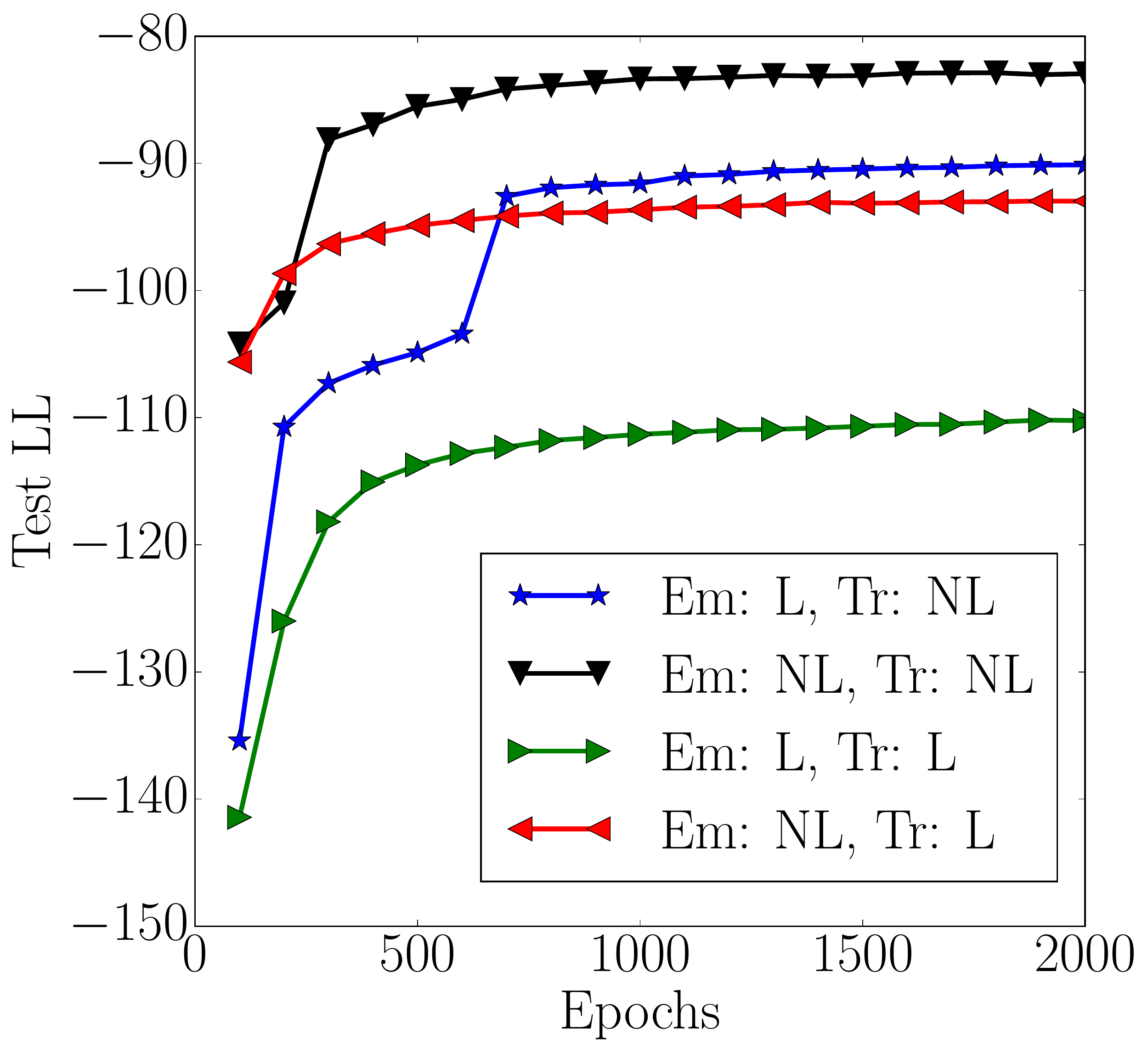}
	\label{fig:medical_ll}
	\caption{\small{Test Log-Likelihood}}
\end{subfigure}
\end{minipage}
\begin{minipage}{0.59\textwidth}
\begin{subfigure}[b]{\linewidth}
	\includegraphics[width=\linewidth]{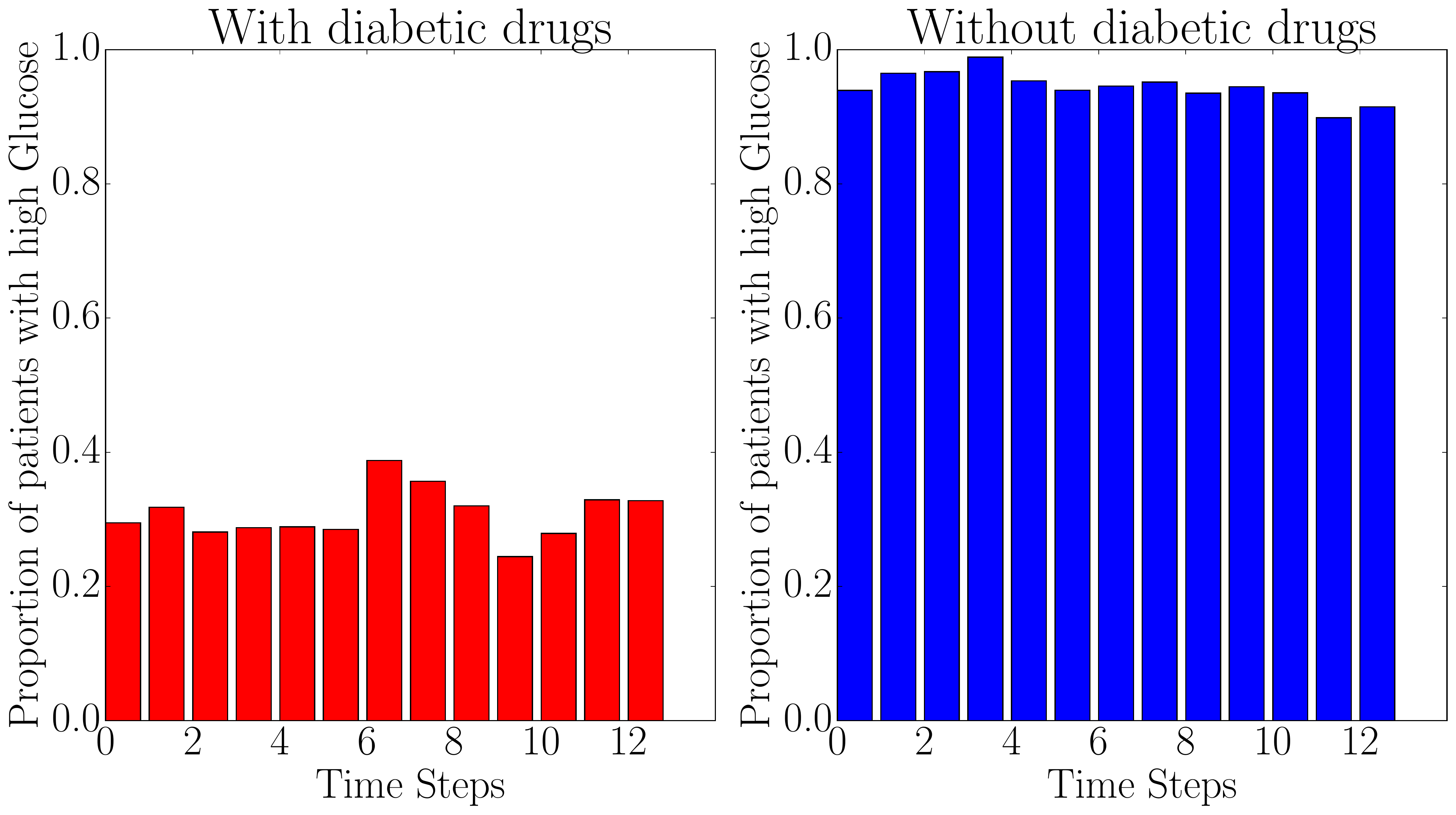}
	\caption{}
	\label{fig:medical_cfac_gluc}
\end{subfigure}\\[1ex]
\begin{subfigure}[b]{\linewidth}
	\includegraphics[width=\linewidth]{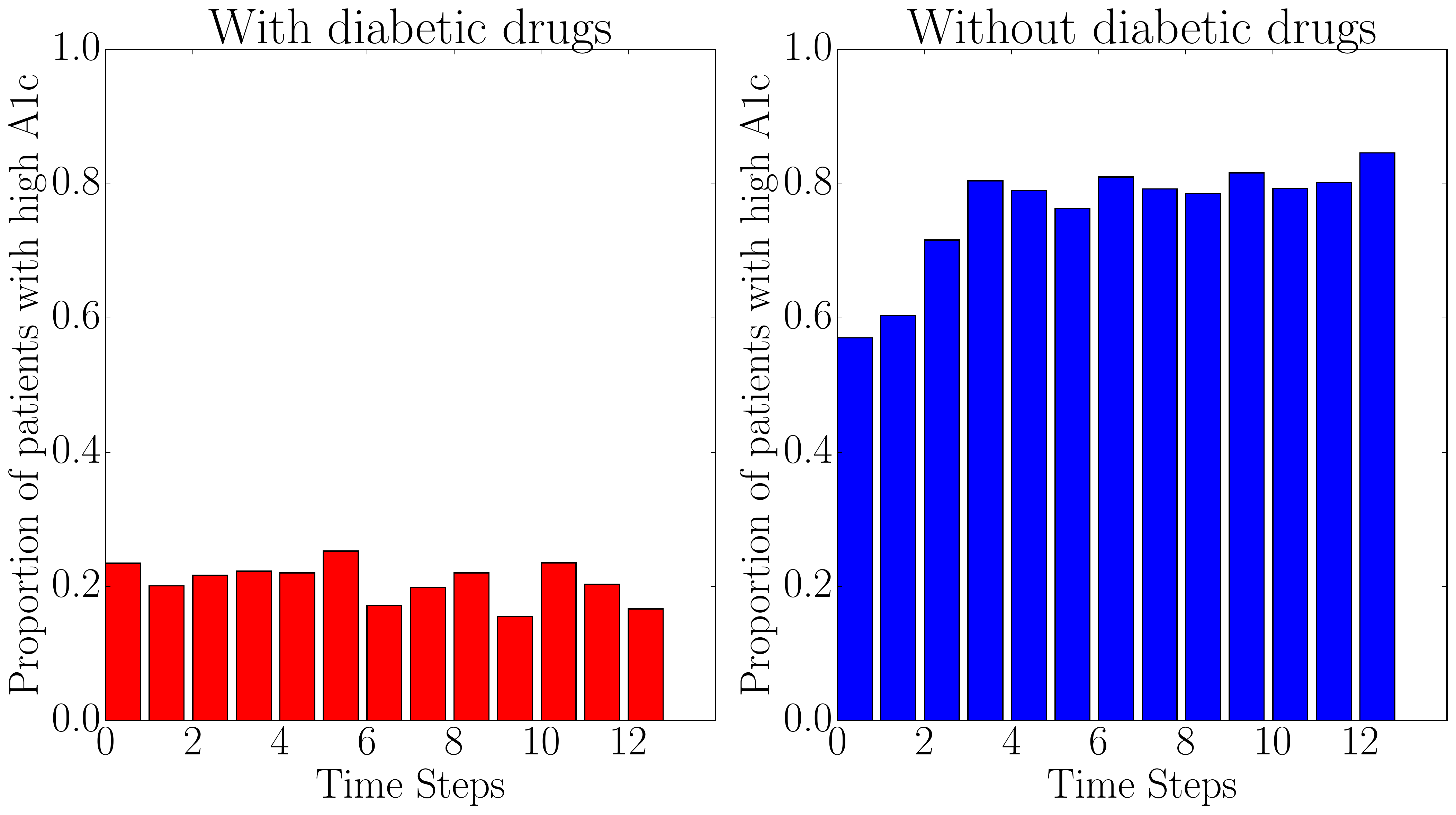}
	\caption{}
	\label{fig:medical_cfac_a1c}
\end{subfigure}
\end{minipage}
	\vskip -3pt\caption{\small{Results of disease progression
            modeling for 8000 diabetic and pre-diabetic patients. (a)
            Test log-likelihood under different model
            variants. Em(ission) denotes $\emisfxn$, Tr(ansition)
            denotes $\meanfxn$ under Linear (L) and Non-Linear (NL)
            functions. We learn a fixed diagonal $\covfxn$. (b)
            Proportion of patients inferred to have high (top
            quantile) glucose level with and without anti-diabetic
            drugs, starting from the time of first Metformin
            prescription. (c) Proportion of patients inferred to have
            high (above 8\%) A1c level with and without anti-diabetic
            drugs, starting from the time of first Metformin
            prescription. Both (b) and (c) are created using the model trained with non-linear emission and transition functions.}}
\label{fig:medical} 
\vspace{-4mm}
\end{figure}
We learn a generative model on healthcare claims data from a major health insurance provider. We look into the effect of anti-diabetic drugs on a population of $8000$ diabetic and pre-diabetic patients.

The (binary) observations of interest here are: A1c level (hemoglobin A1c, a type of protein commonly used in the medical literature to indicate level of diabetes where high A1c level are an indicator for diabetes) and glucose (the amount of a patient's blood sugar).
We bin glucose into quantiles and A1c into medically meaningful bins.
The observations also include age, gender and ICD-9 diagnosis codes depicting various comorbidities of diabetes such as congestive heart failure, chronic kidney disease and obesity.

For actions, we consider prescriptions of nine diabetic drugs including Metformin and Insulin, where Metformin is the most commonly prescribed first-line anti-diabetic drug. For each patient, we group their data over four and half years into
three months intervals.  

We aim to assess the effect of anti-diabetic drugs on a patient's A1c and glucose levels. To that end, we ask a counterfactual question: how would the patient's A1c and glucose levels be had they not received the prescribed medications as observed in the dataset.

A complication in trying to perform counterfactual inference for the A1c and glucose levels is that these quantities are not always measured for each patient at each time step. Moreover, patients who are suspected of being diabetic are tested much more often for their A1c and glucose levels, compared with healthy patients, creating a confounding factor, since diabetic patients tend to have higher A1c and glucose levels. To overcome this we add an observation variable called ``lab indicator'', denoted $x^{\text{ind}}$, which indicates whether the respective lab test, either A1c or glucose, was taken regardless of its outcome. We condition  the time $t$ lab indicator observation, ${x_t}^{\text{ind}}$, on the latent state $z_t$,  and we condition the A1c and glucose value observations on both the latent state and the lab indicator observation. That way, once the model is trained we can perform counterfactual inference by using the \emph{do}-operator on the lab indicator: setting it to $1$ and ignoring its dependence on the latent state. See Figure \ref{fig:labind_both} for an illustration of the model.

We train the model on a dataset of $8000$ patients. We use $\qBRNN$ as the recognition model.

\vspace{1mm}
\textbf{Variants of Kalman Filters: }
Figure \ref{fig:medical}(a) depicts the test log likelihood under variants of the graphical model depicted in Figure \ref{fig:dkf}. \textbf{Em}(ission) denotes $\emisfxn$, the emission function,  and \textbf{Tr}(ansition) denotes $\meanfxn$, the transition function of the mean. 
We learn a fixed diagonal covariance matrix ($\covfxn$). See Eq. \eqref{eqn:gen_model}
for the role these quantities play in the generative model.
Linear (L) denotes a linear relationship between entities, 
and Non-Linear (NL) denotes a non-linear one parameterized by a two-layer neural network. Early in training, a non-linear emission function suffices to achieve a high test log likelihood though as training progresses
the models with non-linear transition functions dominate.

\vspace{1mm}
\textbf{Counterfactual Inference: }
We use a model with non-linear transition and non-linear emission functions.
We look at patients whose first prescribed anti-diabetic drug was Metformin, the most common first-line anti-diabetic drug, and who have at least six months of data before the first Metformin prescription. This leaves us with 800 patients for whom we ask the counterfactual question. For these patients, we infer the latent state up to the time $t_0$ of first Metformin prescription. After $t_0$ we perform forward sampling under two conditions: the ``{\bf{with}}'' condition is using the patient's true medication prescriptions; the ``{\bf{without}}'' condition is removing the medication prescriptions, simulating a patient who receives no anti-diabetic medication. In both cases we set the lab indicator variable ${x_t}^{\text{ind}}$ to be $1$, so we can observe the A1c and glucose lab values. We then compare the inferred A1c and glucose lab values between the ``{\bf{with}}'' and ``{\bf{without}}'' conditions after the time of first Metformin prescription. Figure \ref{fig:medical} presents the results, where we track the proportion of patients with high glucose level (glucose in the top quantile) and high A1c levels (A1c above 8\%), starting from the time of first Metformin prescription. It is evident that patients who do not receive the anti-diabetic drugs are much more prone to having high glucose and A1c levels.
\begin{figure}
    \centering
    \begin{subfigure}[t]{0.45\textwidth}
    \centering
	\begin{tikzpicture}[scale=1, transform shape]
		\node[obs] (xt) {$\mathbf{x_t}$};
		\node[latent, above=of xt] (zt) {$\mathbf{z_t}$};
		\node[latent, left=of zt, xshift=-0.2cm] (z1) {$\mathbf{z_{t-1}}$};
		\node[latent, right=of zt, xshift=0.8cm] (zT) {$\mathbf{z_{t+1}}$};
		\node[obs, below=of zT, yshift=1cm, xshift=-1.3cm] (xind) {$\mathbf{{x_t}^{\text{ind}}}$};		
		\edge {zt} {xt};
		\edge {z1} {zt};
		\edge {zt}{zT};
		\edge{zt}{xind};
		\edge{xind}{xt};
		\plate [xscale=1] {} {(xt)(zt)(xind)} {} ;
	\end{tikzpicture}
	\label{fig:labind}
	\caption{\small{Graphical model during training}}
   \end{subfigure}
    \begin{subfigure}[t]{0.4\textwidth}
    \centering
	\begin{tikzpicture}[scale=1, transform shape]
		\node[obs] (xt) {$\mathbf{x_t}$};
		\node[latent, above=of xt] (zt) {$\mathbf{z_t}$};
		\node[latent, left=of zt, xshift=-0.2cm] (z1) {$\mathbf{z_{t-1}}$};
		\node[latent, right=of zt, xshift=0.8cm] (zT) {$\mathbf{z_{t+1}}$};
		\tikzstyle{obs_do}=[circle,fill=gray!85,draw=black,minimum size=17pt,inner sep=2pt]
		\node[obs_do, below=of zT, yshift=1cm, xshift=-1.3cm] (xind) {$\mathbf{1}$};	
		\edge {zt} {xt};
		\edge {z1} {zt};
		\edge {zt}{zT};
		\edge{xind}{xt};
		\plate [xscale=1] {} {(xt)(zt)(xind)} {} ;
	\end{tikzpicture}
	\label{fig:labind_do}
	\caption{\small{Graphical model during counterfactual inference}}
   \end{subfigure}
   \caption{\small{(a) Generative model with lab indicator variable, focusing on time step $t$. (b) For counterfactual inference we set ${x_t}^{\text{ind}}$ to 1, implementing Pearl's \emph{do}-operator}}
   \label{fig:labind_both}
\end{figure}
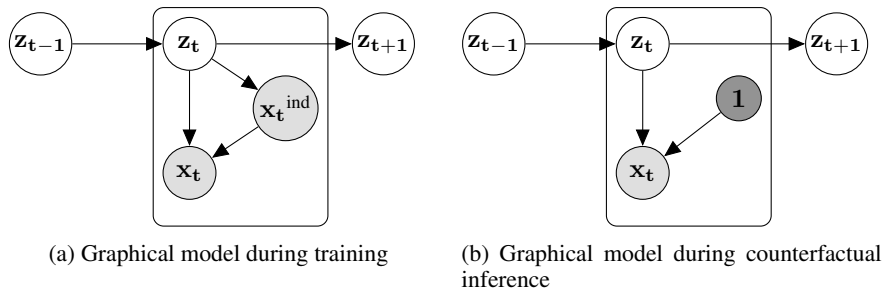

%% file: discussion.tex
\section{Discussion}

We show promising results that nonlinear-state space models can be effective
models for counterfactual analysis. 
The parametric posterior
can be used to approximate the latent state of unseen data. 
We can 
forward sample from the model under different actions and observe their consequent effect.  
Beyond counterfactual inference, the model represents a natural way 
to embed patients into latent space making it possible to ask questions about patient similarity.
Another avenue of work is understanding whether the latent variable space encodes identifiable
characteristics of a patient and whether the evolution of the latent space corresponds to known disease trajectories. 

There exists interesting avenues of future work for our model in a multitude of areas. 
A natural question that arises, particularly with models trained on the Healing MNIST
is the quality of temporal and spatial invariance in the learned filters. Unsupervised learning
of videos is another domain where our model holds promise. Approaches such as \citep{srivastava2015unsupervised}
model video sequences using LSTMs with deterministic transition operators. The effect of stochasticity
in the latent space is an interesting one to explore.

%% file: app_likelihood.tex
\section{Lower Bound on Likelihood\label{appsec:lb_likelihood}}

In the following we omit the dependence of $q$ on $\vecx$ and $\vecu$, and omit the subscript $\phi$.
We can show that the $\KL$ divergence between the approximation to the posterior and the prior simplifies as:
\begin{dmath}
	\KL(q(z_1,\ldots,z_T)||p(z_1,\ldots,z_T))
	= \int_{z_1}\ldots\int_{z_T} q(z_1)\ldots q(z_T|z_{T-1}) \log \frac{p(z_1,z_2,\ldots,z_T)}{q(z_1)\ldots q(z_T|z_{T-1})}\\
	\mathtext{Factorization of the variational distribution}\\ 
	= \int_{z_1}\ldots\int_{z_T} q(z_1)\ldots q(z_T|z_{T-1})
	\log \frac{p(z_1) p(z_2|z_1,u_1)\ldots p(z_T|z_{T-1},u_{T-1})}{q(z_1)\ldots q(z_T|z_{T-1})}\\ 
	\mathtext{Factorization of the prior}\\
	= \int_{z_1}\ldots\int_{z_T} q(z_1)\ldots q(z_T|z_{T-1}) \log \frac{p(z_1)}{q(z_1)}
	+ \sum_{t=2}^{T} \int_{z_1}\ldots\int_{z_T} q(z_1)\ldots q(z_T|z_{T-1}) \log\frac{p(z_t|z_{t-1})}{q(z_t|z_{t-1})}\\
	= \int_{z_1} q(z_1)\log\frac{p(z_1)}{q(z_1)} +  \sum_{t=2}^{T} \int_{z_{t-1}}\int_{z_{t}}q(z_t)\log\frac{p(z_t|z_{t-1})}{q(z_t|z_{t-1})}\\
	\mathtext{Each expectation over $z_t$ is constant for $t\notin\{t,t-1\}$}\\
	= \KL(q(z_1)||p(z_1)) 
	+ \sum_{t=2}^{T-1} \Exp{\KL(q(z_t|z_{t-1})||p(z_t|z_{t-1},u_{t-1}))}{q(z_{t-1})}\\
\end{dmath}

For evaluating the marginal likelihood on the test set, we can use the following Monte-Carlo estimate:
\begin{dmath}
	p(\vecx)\approxeq \frac{1}{S} \sum_{s=1}^S \frac{p(\vecx|\vecz^{(s)}) p(\vecz^{(s)}) }{q(\vecz^{(s)}|\vecx)} \;\;\; \vecz^{(s)}\sim q(\vecz|\vecx)
\end{dmath}
This may be derived in a manner akin to the one depicted in Appendix E \citep{rezende2014stochastic} or Appendix D \citep{kingma2013auto}.

The log likelihood on the test set is computed using: 
\begin{dmath}
	\label{eqn:test_ll_logsum}
	\log p(\vecx)\approxeq \log \frac{1}{S} \sum_{s=1}^S \exp \log\left[\frac{p(\vecx|\vecz^{(s)}) p(\vecz^{(s)}) }{q(\vecz^{(s)}|\vecx)}\right]
\end{dmath}
\eqref{eqn:test_ll_logsum} may be computed in a numerically stable manner using the log-sum-exp trick.  

%% file: factorized_kl.tex
\section{KL divergence between Prior and Posterior\label{appsec:kldiv}}

Maximum likelihood learning requires us to compute:
\begin{dmath}
	\label{eqn:KLdiv_factorized_app}
	KL(q(z_1,\ldots,z_T)||p(z_1,\ldots,z_T))
	= KL(q(z_1)||p(z_1)) + \sum_{t=2}^{T-1} \Exp{KL(q(z_t)||p(z_t|z_{t-1},u_{t-1}))}{q(z_{t-1})}\\
\end{dmath}

The KL divergence between two multivariate Gaussians $q$, $p$ with respective means and covariances $\mu_q, \Sigma_q, \mu_p, \Sigma_p$ can be written as:
\begin{dmath}
	\label{eqn:KLdiv_multivar}
	KL(q||p) = \frac{1}{2}\left( \underbrace{\log\frac{|\Sigma_p|}{|\Sigma_q|}}_{(a)} -D + 
	\underbrace{\Tr(\Prec{p}\Sigma_q)}_{(b)} + \underbrace{(\mu_p-\mu_q)^T\Prec{p}(\mu_p-\mu_q)}_{(c)}\right)
\end{dmath}
The choice of $q$ and $p$ is suggestive. using \eqref{eqn:KLdiv_factorized_app} \& \eqref{eqn:KLdiv_multivar}, 
we can derive a closed form for the KL divergence between $q(z_1\ldots z_T)$ and $p(z_1\ldots z_T)$.

$\mu_q,\Sigma_q$ are the outputs of the variational model. Our functional form for $\mu_p,\Sigma_p$ is based on our generative and can
be summarized as: 
\begin{align*}
	\mu_{p1} = 0\qquad
	\Sigma_{p1} = \Id\qquad
	\mu_{pt} = G(z_{t-1},u_{t-1}) = G_{t-1}\qquad
	\Sigma_{pt} = \dt\vsigma\qquad
\end{align*}

Here, $\Sigma_{pt}$ is assumed to be a learned diagonal matrix and $\dt$ a scalar parameter.  

\textbf{Term (a)}
For $t=1$, we have:
\begin{dmath}
	\label{eqn:logdet_1}
	\log\frac{|\Sigma_{p1}|}{|\Sigma_{q1}|} = \log|\Sigma_{p1}|-\log|\Sigma_{q1}|
	= -\log|\Sigma_{q1}|
\end{dmath}

For $t>1$, we have:
\begin{dmath}
	\label{eqn:logdet_t}
	\log\frac{|\Sigma_{pt}|}{|\Sigma_{qt}|} = \log|\Sigma_{pt}|-\log|\Sigma_{qt}|
	= D \log(\dt) + \log|\vsigma| -\log|\Sigma_{qt}|
\end{dmath}

\textbf{Term (b)}
For $t=1$, we have:
\begin{dmath}
	\label{eqn:trace_1}
	\Tr(\Prec{p1}\Sigma_{q1}) = \Tr(\Sigma_{q1})
\end{dmath}

For $t>1$, we have: 
\begin{dmath}
	\label{eqn:trace_t}
	\Tr(\Prec{pt}\Sigma_{qt}) = \frac{1}{\dt}\Tr(\diag(\vsigma)^{-1}\Sigma_{qt})
\end{dmath}

\textbf{Term (c)}
For $t=1$, we have:
\begin{equation}
	\label{eqn:quad_form_1}
	(\mu_{p1}-\mu_{q1})^T\Sigma_{p1}^{-1}(\mu_{p1}-\mu_{q1}) = ||\mu_{q1}||^2\\
\end{equation}

For $t>1$, we have:
\begin{equation}
	\label{eqn:quad_form_t}
	(\mu_{pt}-\mu_{qt})^T\Sigma_{pt}^{-1}(\mu_{pt}-\mu_{qt}) = \\
	\dt (G_{t-1}-\mu_{qt})^T\diag(\vsigma)^{-1}(G_{t-1}-\mu_{qt})
\end{equation}

Rewriting \eqref{eqn:KLdiv_factorized_app} using \eqref{eqn:logdet_1}, \eqref{eqn:logdet_t}, \eqref{eqn:trace_1}, \eqref{eqn:trace_t}, \eqref{eqn:quad_form_1}, \eqref{eqn:quad_form_t}, we get:

\begin{dmath}
	\label{eqn:KLdiv_final}
	KL(q(z_1,\ldots,z_T)||p(z_1,\ldots,z_T)) = 
	\frac{1}{2}\left((T-1)D \log(\dt)\log|\vsigma| -\sum_{t=1}^T\log|\Sigma_{qt}|  \\
	+ \Tr(\Sigma_{q1})+\frac{1}{\dt}\sum_{t=2}^T \Tr(\diag(\vsigma)^{-1}\Sigma_{qt})
	+ ||\mu_{q1}||^2 \\
	+ \dt \sum_{t=2}^T \Exp{(G_{t-1}-\mu_{qt})^T\diag(\vsigma)^{-1}(G_{t-1}-\mu_{qt})}{z_{t-1}}\right)\\
\end{dmath}

We can now take gradients with respect to $\mu_{qt}, \Sigma_{qt}, G(z_{t-1},u_{t-1})$ in \eqref{eqn:KLdiv_final}.

%% file: add_expts.tex
\section{Additional Experimental Results}
We consider a variant of {\bf{Large Healing MNIST}} trained on $100$ different kinds of $0,2$s each.
\begin{figure}[h]
\begin{subfigure}[b]{0.45\textwidth}
	\centering
	\includegraphics[width=0.65\textwidth]{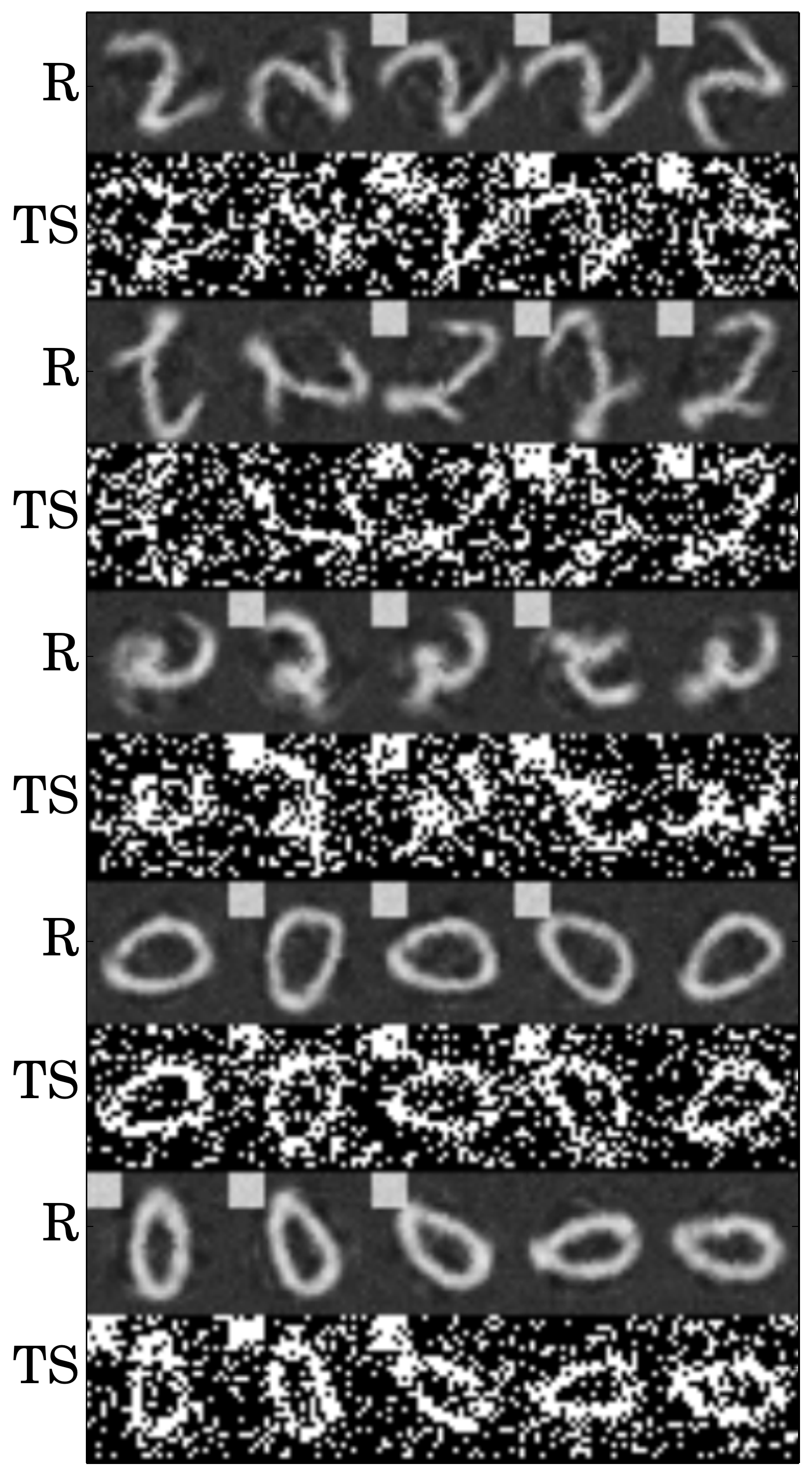}
	\caption{\small{{\bf{Large Healing MNIST}} (0,2):Pairs of Training Sequences (TS) and mean probabilities of Reconstructions (R) shown above.}}
	\label{fig:healingMNIST_reconstructions02}
\end{subfigure}
\begin{subfigure}[b]{0.45\textwidth}
	\centering
	\includegraphics[width=0.55\textwidth]{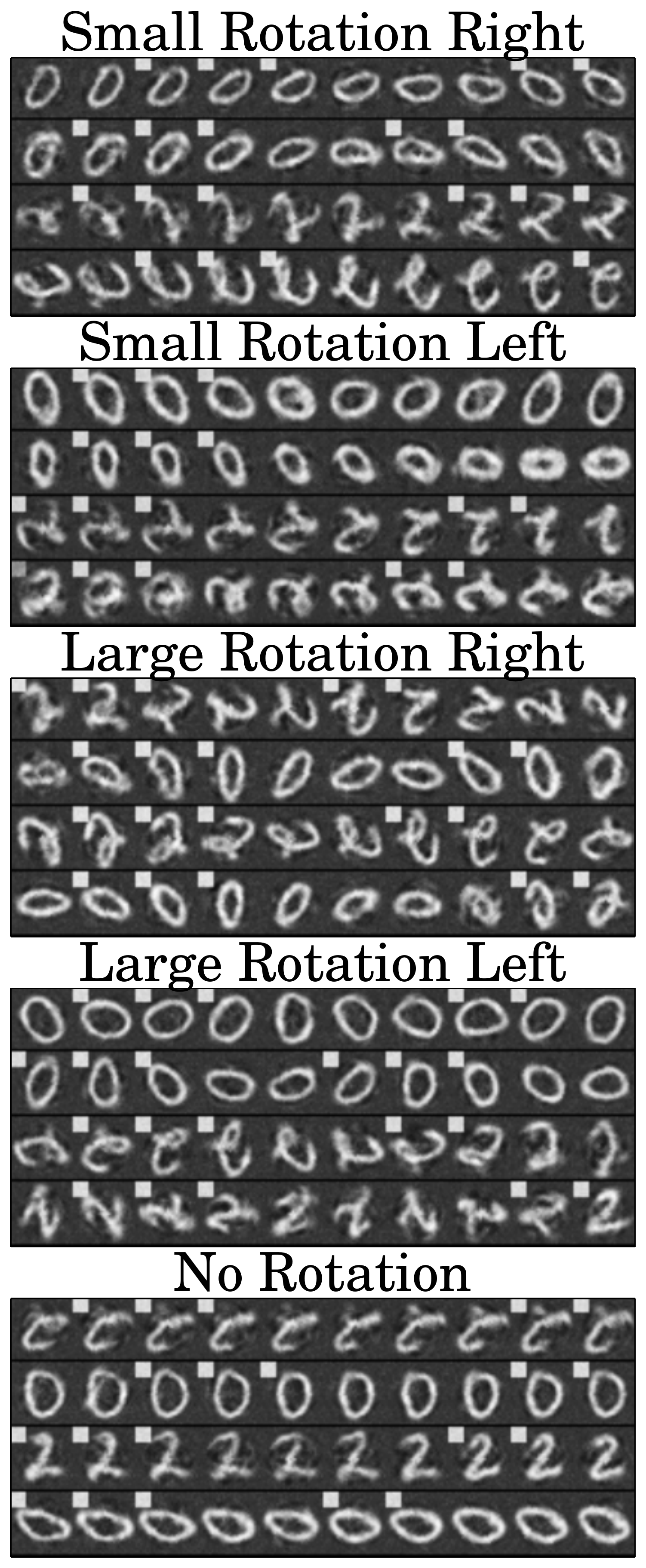}
	\caption{{\bf{Large Healing MNIST}} (0,2): Mean probabilities sampled under different, constant rotations.}
	\label{fig:healingMNIST_samples02}
\end{subfigure}
\end{figure}

Figure \ref{fig:healingMNIST_reconstructions02} and \ref{fig:healingMNIST_samples02} depict the reconstructions and samples from a model trained on the digits $0$ and $2$.

%% file: dkf.bbl
\begin{thebibliography}{}

\bibitem[\protect\citename{Bayer \& Osendorfer, }2014]{bayer2014learning}
Bayer, Justin, \& Osendorfer, Christian. 2014.
\newblock Learning Stochastic Recurrent Networks.
\newblock {\em arXiv preprint arXiv:1411.7610}.

\bibitem[\protect\citename{Bottou {\em et~al.\ }\relax,
  }2013]{bottou2013counterfactual}
Bottou, L{\'e}on, Peters, Jonas, Quinonero-Candela, Joaquin, Charles, Denis~X,
  Chickering, D~Max, Portugaly, Elon, Ray, Dipankar, Simard, Patrice, \&
  Snelson, Ed. 2013.
\newblock Counterfactual reasoning and learning systems: The example of
  computational advertising.
\newblock {\em The Journal of Machine Learning Research}, {\bf 14}(1),
  3207--3260.

\bibitem[\protect\citename{Chung {\em et~al.\ }\relax,
  }2015]{chung2015recurrent}
Chung, Junyoung, Kastner, Kyle, Dinh, Laurent, Goel, Kratarth, Courville,
  Aaron, \& Bengio, Yoshua. 2015.
\newblock A Recurrent Latent Variable Model for Sequential Data.
\newblock {\em arXiv preprint arXiv:1506.02216}.

\bibitem[\protect\citename{Collobert {\em et~al.\ }\relax,
  }2011]{collobert2011torch7}
Collobert, Ronan, Kavukcuoglu, Koray, \& Farabet, Cl{\'e}ment. 2011.
\newblock Torch7: A matlab-like environment for machine learning.
\newblock {\em In:} {\em BigLearn, NIPS Workshop}.

\bibitem[\protect\citename{Gan {\em et~al.\ }\relax, }2015]{ganSigmoid}
Gan, Zhe, Li, Chunyuan, Henao, Ricardo, Carlson, David, \& Carin, Lawrence.
  2015.
\newblock Deep Temporal Sigmoid Belief Networks for Sequence Modeling.

\bibitem[\protect\citename{Goroshin {\em et~al.\ }\relax,
  }2015]{goroshin2015learning}
Goroshin, Ross, Mathieu, Michael, \& LeCun, Yann. 2015.
\newblock Learning to Linearize under Uncertainty.
\newblock {\em arXiv preprint arXiv:1506.03011}.

\bibitem[\protect\citename{Gregor {\em et~al.\ }\relax, }2015]{gregor2015draw}
Gregor, Karol, Danihelka, Ivo, Graves, Alex, Rezende, Danilo~Jimenez, \&
  Wierstra, Daan. 2015.
\newblock {DRAW:} {A} Recurrent Neural Network For Image Generation.
\newblock {\em In:} {\em Proceedings of the 32nd International Conference on
  Machine Learning, {ICML} 2015, Lille, France, 6-11 July 2015}.

\bibitem[\protect\citename{Haykin, }2004]{haykin2004kalman}
Haykin, Simon. 2004.
\newblock {\em Kalman filtering and neural networks}.
\newblock  Vol. 47.
\newblock John Wiley \& Sons.

\bibitem[\protect\citename{H{\"o}fler, }2005]{hofler2005causal}
H{\"o}fler, M. 2005.
\newblock Causal inference based on counterfactuals.
\newblock {\em BMC medical research methodology}, {\bf 5}(1), 28.

\bibitem[\protect\citename{Jazwinski, }2007]{jazwinski2007stochastic}
Jazwinski, Andrew~H. 2007.
\newblock {\em Stochastic processes and filtering theory}.
\newblock Courier Corporation.

\bibitem[\protect\citename{Kalman, }1960]{kalman1960new}
Kalman, Rudolph~Emil. 1960.
\newblock A new approach to linear filtering and prediction problems.
\newblock {\em Journal of Fluids Engineering}, {\bf 82}(1), 35--45.

\bibitem[\protect\citename{Kingma \& Ba, }2014]{kingma2014adam}
Kingma, Diederik, \& Ba, Jimmy. 2014.
\newblock Adam: A method for stochastic optimization.
\newblock {\em arXiv preprint arXiv:1412.6980}.

\bibitem[\protect\citename{Kingma \& Welling, }2013]{kingma2013auto}
Kingma, Diederik~P, \& Welling, Max. 2013.
\newblock Auto-encoding variational bayes.
\newblock {\em arXiv preprint arXiv:1312.6114}.

\bibitem[\protect\citename{Langford {\em et~al.\ }\relax,
  }2009]{langford2009learning}
Langford, John, Salakhutdinov, Ruslan, \& Zhang, Tong. 2009.
\newblock Learning Nonlinear Dynamic Models.
\newblock {\em Pages  593--600 of:} {\em Proceedings of the 26th Annual
  International Conference on Machine Learning}.
\newblock ACM.

\bibitem[\protect\citename{LeCun \& Cortes, }2010]{lecun2010mnist}
LeCun, Yann, \& Cortes, Corinna. 2010.
\newblock MNIST handwritten digit database.
\newblock {\em AT\&T Labs [Online]. Available: http://yann. lecun.
  com/exdb/mnist}.

\bibitem[\protect\citename{Memisevic, }2013]{memisevic2013learning}
Memisevic, Roland. 2013.
\newblock Learning to Relate Images.
\newblock {\em Pattern Analysis and Machine Intelligence, IEEE Transactions
  on}, {\bf 35}(8), 1829--1846.

\bibitem[\protect\citename{Mirowski \& LeCun, }2009]{mirowski2009dynamic}
Mirowski, Piotr, \& LeCun, Yann. 2009.
\newblock Dynamic factor graphs for time series modeling.
\newblock {\em Pages  128--143 of:} {\em Machine Learning and Knowledge
  Discovery in Databases}.
\newblock Springer.

\bibitem[\protect\citename{Morgan \& Winship, }2014]{morgan2014counterfactuals}
Morgan, Stephen~L, \& Winship, Christopher. 2014.
\newblock {\em Counterfactuals and causal inference}.
\newblock Cambridge University Press.

\bibitem[\protect\citename{Pearl, }2009]{pearl2009causality}
Pearl, Judea. 2009.
\newblock {\em Causality}.
\newblock Cambridge university press.

\bibitem[\protect\citename{Raiko \& Tornio, }2009]{raiko2009variational}
Raiko, Tapani, \& Tornio, Matti. 2009.
\newblock Variational Bayesian learning of nonlinear hidden state-space models
  for model predictive control.
\newblock {\em Neurocomputing}, {\bf 72}(16), 3704--3712.

\bibitem[\protect\citename{Rezende \& Mohamed, }2015]{rezende2015variational}
Rezende, Danilo~Jimenez, \& Mohamed, Shakir. 2015.
\newblock Variational Inference with Normalizing Flows.
\newblock {\em arXiv preprint arXiv:1505.05770}.

\bibitem[\protect\citename{Rezende {\em et~al.\ }\relax,
  }2014]{rezende2014stochastic}
Rezende, Danilo~Jimenez, Mohamed, Shakir, \& Wierstra, Daan. 2014.
\newblock Stochastic backpropagation and approximate inference in deep
  generative models.
\newblock {\em arXiv preprint arXiv:1401.4082}.

\bibitem[\protect\citename{Rosenbaum, }2002]{rosenbaum2002observational}
Rosenbaum, Paul~R. 2002.
\newblock {\em Observational studies}.
\newblock Springer.

\bibitem[\protect\citename{Roweis \& Ghahramani, }2000]{roweis2000algorithm}
Roweis, Sam, \& Ghahramani, Zoubin. 2000.
\newblock An EM algorithm for identification of nonlinear dynamical systems.

\bibitem[\protect\citename{Srivastava {\em et~al.\ }\relax,
  }2014]{srivastava2014dropout}
Srivastava, Nitish, Hinton, Geoffrey, Krizhevsky, Alex, Sutskever, Ilya, \&
  Salakhutdinov, Ruslan. 2014.
\newblock Dropout: A simple way to prevent neural networks from overfitting.
\newblock {\em The Journal of Machine Learning Research}, {\bf 15}(1),
  1929--1958.

\bibitem[\protect\citename{Srivastava {\em et~al.\ }\relax,
  }2015]{srivastava2015unsupervised}
Srivastava, Nitish, Mansimov, Elman, \& Salakhutdinov, Ruslan. 2015.
\newblock Unsupervised learning of video representations using LSTMs.
\newblock {\em arXiv preprint arXiv:1502.04681}.

\bibitem[\protect\citename{Tabak {\em et~al.\ }\relax, }2010]{tabak2010density}
Tabak, Esteban~G, Vanden-Eijnden, Eric, {\em et~al.\ }\relax. 2010.
\newblock Density estimation by dual ascent of the log-likelihood.
\newblock {\em Communications in Mathematical Sciences}, {\bf 8}(1), 217--233.

\bibitem[\protect\citename{Valpola \& Karhunen, }2002]{valpola2002unsupervised}
Valpola, Harri, \& Karhunen, Juha. 2002.
\newblock An unsupervised ensemble learning method for nonlinear dynamic
  state-space models.
\newblock {\em Neural computation}, {\bf 14}(11), 2647--2692.

\bibitem[\protect\citename{Velez, }2013]{finale2013pomdp}
Velez, Finale~Doshi. 2013.
\newblock Partially-Observable Markov Decision Processes as Dynamical Causal
  Models.

\bibitem[\protect\citename{Wan {\em et~al.\ }\relax, }2000]{wan2000unscented}
Wan, Eric, Van Der~Merwe, Ronell, {\em et~al.\ }\relax. 2000.
\newblock The unscented Kalman filter for nonlinear estimation.
\newblock {\em Pages  153--158 of:} {\em Adaptive Systems for Signal
  Processing, Communications, and Control Symposium 2000. AS-SPCC. The IEEE
  2000}.
\newblock IEEE.

\bibitem[\protect\citename{Watter {\em et~al.\ }\relax, }2015]{watter2015embed}
Watter, Manuel, Springenberg, Jost~Tobias, Boedecker, Joschka, \& Riedmiller,
  Martin. 2015.
\newblock Embed to Control: A Locally Linear Latent Dynamics Model for Control
  from Raw Images.
\newblock {\em arXiv preprint arXiv:1506.07365}.

\bibitem[\protect\citename{Zaremba \& Sutskever, }2014]{zaremba2014learning}
Zaremba, Wojciech, \& Sutskever, Ilya. 2014.
\newblock Learning to Execute.
\newblock {\em arXiv preprint arXiv:1410.4615}.

\end{thebibliography}
